%% file: main_bibtex_ready.tex
\documentclass{article}
\usepackage[utf8]{inputenc}
\usepackage{amsmath, amssymb, amsthm, fullpage}
\usepackage{caption}
\usepackage{enumitem}
\usepackage{graphicx}
\usepackage{xcolor}
\usepackage{tikz}
\usetikzlibrary{positioning,arrows.meta,calc}
\usepackage{booktabs}
\newtheorem{theorem}{Theorem}
\newtheorem{lemma}{Lemma}
\newtheorem{proposition}{Proposition}
\newtheorem{corollary}{Corollary}
\theoremstyle{definition}
\newtheorem{definition}{Definition}
\newtheorem{assumption}{Assumption}
\theoremstyle{remark}
\newtheorem{remark}{Remark}

\title{When Can Safe Controllers Adapt?\\ {\Large Information before Commitment}}
\author{Venkatesh Saligrama\thanks{Department of Electrical and Computer Engineering, Boston University, MA 02215, Email: srv@bu.edu}}
\date{}

\begin{document}
\maketitle
\input{safe_v15_package/abstract_v11_constraints_bite}

\input{safe_v15_package/introduction_v23_constraints_bite}
\input{safe_v15_package/setup_v14_edited_direct}

\input{safe_v15_package/mechanisms_v5_edited}

\input{safe_v15_package/positive_v5_edited}
\input{safe_v15_package/computable_v3_edited_final}
\input{safe_v15_package/numerical_v3_edited}
\input{safe_v15_package/related_v4_constraints_bite}

\input{safe_v15_package/discussion_v3_constraints_bite}
\appendix
\input{safe_v15_package/appendix_mechanisms_edited}
\input{safe_v15_package/appendix_positive_v2_edited}
\input{safe_v15_package/appendix_computable_edited}
\bibliographystyle{unsrt}
\bibliography{safe_v15_package/safe_learning_refs_v5}

\end{document}

%% file: safe_v15_package/abstract_v11_constraints_bite.tex
\begin{abstract}
Safe adaptive control is online adaptation under a safety guarantee on the
learning trajectory itself. The controller may use any causal,
history-dependent rule and act differently across environments as data
arrive. Only its safety guarantee is uniform: the same rule must satisfy it
under every initially plausible model. Performance is measured against a safe
oracle that knows the realized model.

Many finite-time analyses assume persistent excitation of the uniformly safe
closed loop, so the data distinguish every pair of models requiring different
control decisions. Under that assumption, feasibility is already settled;
only the rate remains. We ask instead:
\emph{Do the safety constraints permit such an informative experiment at all?}

While an alternative remains plausible, the controller must preserve a safe
continuation under it. We call the first action that forecloses such a
continuation \emph{commitment}. Chance safety allows commitment only on an
event rare under the alternative, and the evidence must arrive beforehand:
the observation generated by the committing action is too late. We define
\emph{precommitment information} as the KL divergence between learner-visible
laws stopped before commitment.

Our main result is a causal reduction. The commitment rule determines three system-dependent quantities: (a) the probability that safety permits commitment under the alternative, (b) the target-side cost of remaining noncommittal, (c) the information available when the decision is made. Bounded precommitment information therefore leaves a fixed fraction of the oracle gap unavoidable. If the gap
is $\Omega(T)$, every uniformly safe policy has linear regret. We establish
the obstruction in a constrained linear system with quadratic regulation
cost. With multiple alternatives, finite information limits how small a certified candidate list can be. We also prove recovery when safe experiments are repeatable,
persistently informative, and return to a common continuation state, and
derive semidefinite upper certificates for deterministic linear-Gaussian
systems.
\end{abstract}

%% file: safe_v15_package/introduction_v23_constraints_bite.tex
\section{Introduction} \label{sec:intro}

\noindent
The safe adaptive control problem is regret minimization over uniformly safe
policies.  A causal policy must satisfy the safety specification under every
model not yet ruled out, and its performance is measured against the safe
oracle that knows the true model.  The difficulty is exploration.  The policy
can learn only through actions that remain safe under every plausible model,
yet it must reach the oracle's model-specific behavior from inside that set.

Many finite-time analyses impose conditions under which this common safe
regime remains informative.  Excitation can be injected, or process noise and
safety-induced state variation keep the closed loop informative
\cite{AbbasiYadkori11,DeanConstrained19,LiDasShammaLi21,%
SchifferJanson25,HutchinsonJiangAlizadeh26}.  Safety then limits where the
system may go, but not the decision-relevant information it can acquire. Under such conditions, familiar adaptive-control rates can be recovered.

We ask what can be said when the constraints also limit the experiment.  This
is not merely a slower estimation problem.  An action that distinguishes
models may consume safety margin under a model that is still plausible.  The
admissible experiment set is then endogenous: it
depends on the margin already spent and on the models not yet falsified.  The
model-specific action sought by adaptation is unavailable until the evidence
supporting it has been collected. This creates an ordering constraint:
evidence must precede commitment, while the same margin that commitment
consumes may also limit the available evidence. The full controlled trajectory is therefore the wrong statistical experiment: it can contain evidence generated by the very action that already required the model-specific decision. The question is no longer
only how fast uncertainty shrinks, but whether enough evidence can be collected
before the first action that forecloses a safe continuation under an
alternative.  We call the evidence available before that action
\emph{precommitment information}.  Bounded precommitment information certifies
an unavoidable oracle gap.  Persistent information generated by repeatable
safety-preserving experiments supports recovery.

\noindent\textbf{Guarantee or exclude.}
With uncertainty unresolved, a uniformly safe controller has two options.  It
can guarantee: remain in the regime that is feasible under every plausible
model and accept the value this assures.  In our lower-bound examples, this
regime already attains the max--min value.  Or it can exclude: once the
evidence permits, take an action as if an alternative were no longer in play.
The exclusion is statistical, not literal.  The safety guarantee remains
uniform over the initial model class; chance safety permits an action that is
unsafe under an alternative only on an event whose probability under that
alternative is at most $\eta_T$.  Making commitment likely under the target
while keeping it this rare under the alternative requires order
$\log(1/\eta_T)$ nats of evidence, and that evidence must be available before
the committing action is chosen.  In the lower-bound examples, adaptation
beyond the max--min value is precisely such an exclusion.  The max--min
criterion does not distinguish the two options once the worst-case value is
fixed.  Oracle regret does.

\noindent\textbf{Oracle regret.}
Let the unknown environment be $\theta\in\Theta_0$.  Over a horizon $T$, let
$\Pi_{\mathrm{safe}}^T(\Theta_0)$ be the causal policies whose probability of
any safety violation is at most $\eta_T$ under every $\theta\in\Theta_0$.
Let $J_T^\star(\theta)$ be the value of the safe oracle that knows $\theta$,
and define
\[
    R_T(\pi;\theta):=J_T^\star(\theta)-J_T(\pi;\theta).
\]
We study
\begin{equation}
    \mathfrak R_T(\Theta_0)
    :=
    \inf_{\pi\in\Pi_{\mathrm{safe}}^T(\Theta_0)}
    \sup_{\theta\in\Theta_0}
    \bigl[J_T^\star(\theta)-J_T(\pi;\theta)\bigr],
    \label{eq:intro-minimax-regret}
\end{equation}
and ask when it is $o(T)$.  More generally, if the oracle advantage itself is
$G_T=o(T)$, we ask whether the learner can close all but $o(G_T)$ of that gap.

\noindent\textbf{Commitment.}
Fix a target environment $\theta$ and an incompatible alternative $\theta'$.
A commitment rule specifies, at each time $t$, a predictable set
$\mathcal D_t^{\theta\mid\theta'}$ of history--action pairs.  The induced
commitment time and event are
\[
    \tau^\pi
    :=\inf\bigl\{t\le T:(H_{t-1},U_t)
        \in\mathcal D_t^{\theta\mid\theta'}\bigr\},
    \qquad
    A^\pi:=\{\tau^\pi\le T\}.
\]
In deterministic constrained systems, commitment is the first action after
which no hard-safe continuation remains under $\theta'$.  Noncommitment is
therefore preservation of recursive feasibility under the alternative.
Because the commitment decision is made from $(H_{t-1},U_t)$, the observation
produced by $U_{\tau^\pi}$ arrives too late to justify that action.

The commitment rule is the fundamental object. It simultaneously determines how often safety permits commitment under the alternative, what noncommitment costs under the target, and how much information is available when the decision is made.  Safety determines how often commitment is allowed under the
alternative.  Value determines the regret cost of not committing under the
target.  Information is measured on the learner-visible record stopped before
commitment.  The controller need not identify the full model; it needs only to
resolve alternatives that require different safe decisions.  This connects to
sequential design and controlled sensing
\cite{Chernoff59,NaghshvarJavidi13,NitinawaratAtiaVeeravalli13}, with the
additional restriction that safety controls which experiments are admissible
and whether their observations arrive before commitment.

\noindent\textbf{Precommitment lower bound.}
Let
\[
    \alpha_T(\theta,\theta')
    :=
    \sup_{\mu\in\Pi_{\mathrm{safe}}^T(\{\theta,\theta'\})}
    \mathbb P_{\theta'}^\mu(A^\mu)
\]
be the largest commitment probability allowed under the alternative.  If
commitment forces a violation under $\theta'$, then
$\alpha_T(\theta,\theta')\le\eta_T$.
For $q\in[0,1]$, define the noncommitment regret profile
\[
    \mathcal G_T(\theta,\theta';q)
    :=
    \inf_{\substack{
      \mu\in\Pi_{\mathrm{safe}}^T(\{\theta,\theta'\})\\
      \mathbb P_\theta^\mu((A^\mu)^c)\ge q
    }}
    R_T(\mu;\theta).
\]
It is the least target regret among uniformly safe policies that do not commit
with probability at least $q$.

Let $I_{\mathrm{pre}}(T;\theta,\theta')$ be the largest KL divergence between
the stopped learner-visible laws under $\theta$ and $\theta'$.  Since
$A^\pi$ is determined by the stopped record, data processing links
precommitment information to the two commitment probabilities.  Theorem~\ref{thm:binary-precommitment-limit}
gives an exact binary-KL bound.  A simpler consequence is
\begin{equation}
    R_T(\pi;\theta)
    \ge
    \mathcal G_T\!\left(
      \theta,\theta';
      \left[
        \frac12e^{-I_{\mathrm{pre}}(T;\theta,\theta')}
        -\alpha_T(\theta,\theta')
      \right]_+
    \right).
    \label{eq:intro-main-profile-bound}
\end{equation}
If a system-specific calculation gives
\[
    \mathcal G_T(\theta,\theta';q)
    \ge [G_Tq-\zeta_T]_+,
\]
then
\begin{equation}
    R_T(\pi;\theta)
    \ge
    \left[
      G_T\left(
        \frac12e^{-I_{\mathrm{pre}}(T;\theta,\theta')}
        -\alpha_T(\theta,\theta')
      \right)
      -\zeta_T
    \right]_+.
    \label{eq:intro-main-bound}
\end{equation}
Thus bounded precommitment information leaves a nonvanishing fraction of the
oracle gap.  If $G_T=\Omega(T)$, every uniformly safe policy has linear oracle
regret.  At the standing risk level $\eta_T=T^{-2}$, the explicit bound remains
positive until the stopped KL reaches $2\log T-\log 2$.  Driving the target
noncommitment probability to zero therefore requires evidence of order
$2\log T$, which a bounded profile cannot supply.

\noindent\textbf{Finite-headroom example.}\quad
A simple example (see Figure~\ref{fig:action-tension}) shows why waiting need not help.  Over $T$ rounds, a
controller chooses commands\footnote{The restriction to nonnegative commands is expositional, not
structural.  One may instead take \(U_t\in[-U,U]\), reward
\(r_t=|U_t|\), and the same thermal dynamics and sensor; both safety
expenditure and Gaussian information depend on \(U_t^2\), so the sign
is immaterial.  A signed aeroelastic realization replaces these by a
noisy flutter signal
\(
Y_t=\theta U_t^2+\xi_t
\)
and a fatigue budget
\(
d_{t+1}=d_t-\theta^2U_t^4.
\)
Then both accumulated fatigue and pairwise information depend on
\(\sum_t U_t^4\), and the same matched-order argument applies.} $U_t\in[0,U]$.  The command earns reward but
consumes thermal headroom:
\[
    d_{t+1}=d_t-\theta u_t^2,\qquad d_t\ge0.
\]
After applying $u_t$, the controller observes
\[
    y_t=\theta u_t+\xi_t,
    \qquad \xi_t\sim\mathcal N(0,\sigma^2).
\]

\begin{figure}[!ht]
\centering
\begin{minipage}[c]{0.35\linewidth}
\centering
\resizebox{\linewidth}{!}{%
\begin{tikzpicture}[
    >=Latex,
    every node/.style={font=\small},
    command/.style={draw, rounded corners, thick, align=center, inner sep=4pt,
        minimum width=1.8cm, minimum height=0.7cm},
    box/.style={draw, rounded corners, align=left, inner sep=3.5pt,
        text width=3.35cm, minimum height=0.72cm},
    lab/.style={font=\scriptsize, align=center}
]
\node[command] (u) at (0,0) {command\\[-1mm]$U_t$};
\node[box] (perf)   at (4.85, 1.05) {\textbf{Performance}\\reward increases};
\node[box] (info)   at (4.85, 0)    {\textbf{Information}\\$\mathrm{kl}_t \propto U_t^2$};
\node[box] (margin) at (4.85,-1.05) {\textbf{Safety margin}\\$d_{t+1}=d_t-\theta U_t^2$};
\draw[->, thick] (u.east) to[out=25,in=180]
    node[midway, above, sloped, lab] {reward} (perf.west);
\draw[->, thick] (u.east) --
    node[midway, above, lab] {signal} (info.west);
\draw[->, thick] (u.east) to[out=-25,in=180]
    node[midway, below, sloped, lab] {consumption} (margin.west);
\end{tikzpicture}}
\end{minipage}%
\hfill
\begin{minipage}[c]{0.58\linewidth}
\centering\small
\setlength{\tabcolsep}{4pt}
\renewcommand{\arraystretch}{1.2}
\resizebox{\linewidth}{!}{%
\begin{tabular}{@{}ll@{}}
\toprule
\textbf{Controller} & \textbf{Reward} \\
\midrule
Healthy oracle & $UT$ \\
Worn-safe max--min baseline &
$\min\{UT,\sqrt{Td_0/\theta_{\rm high}}\}$ \\
Uniformly safe learner & may exceed the baseline only after evidence \\
\bottomrule
\end{tabular}%
}
\end{minipage}

\vspace{2ex}
{\small
\begin{tabular}{@{}p{0.55\linewidth}@{\quad}p{0.4\linewidth}@{}}
Evidence required for error of order $\eta_T=T^{-2}$: &
order $\log(1/\eta_T)=2\log T$ \\
Evidence available before leaving the worn budget: &
$\theta_{\rm high}d_0/(2\sigma^2)$ --- independent of the horizon \\
\end{tabular}}
\caption{\textbf{Finite headroom.}  A command affects reward, information,
and thermal headroom.  The evidence needed to leave the worn-safe regime
grows with the horizon, while the evidence available before leaving it is
bounded.}
\label{fig:action-tension}
\end{figure}

\medskip
\noindent
Take $r_t=u_t$ and compare a healthy device
$\theta_{\rm low}=0$ with a worn device $\theta_{\rm high}>0$.
The healthy oracle uses $u_t=U$ in every round.  Let $\tau^\pi$ be the first
time the learner leaves the worn-device budget.  Before commitment, safety
under the worn model gives
\[
    \sum_{t<\tau^\pi}u_t^2
    \le \frac{d_0}{\theta_{\rm high}}.
\]
The stopped Gaussian KL therefore satisfies
\begin{equation}
    I_{\mathrm{pre}}(T;\theta_{\rm low},\theta_{\rm high})
    \le
    \frac{(\theta_{\rm high}-\theta_{\rm low})^2}{2\sigma^2}
    \frac{d_0}{\theta_{\rm high}}
    =
    \frac{\theta_{\rm high}d_0}{2\sigma^2},
    \label{eq:intro-thermal-expenditure}
\end{equation}
which is independent of $T$.  The required evidence grows as $2\log T$, but
the available evidence does not grow at all.  The healthy oracle gap is
$\Theta(T)$, so the lower bound gives linear regret.

\noindent\textbf{Recovery.}
An unbounded information profile alone does not prove recovery.  One safe
policy must generate the information and then enter the selected operating
regime without losing feasibility.  We give a sufficient condition based on a
repeatable safety-preserving experiment.  In the braking example, maximal
braking is both safe and informative, and a confidence-based controller
achieves $O(\sqrt{T\log T})$ regret.

\paragraph{Contributions.}
\begin{enumerate}[leftmargin=*]
    \item \textbf{Commitment and precommitment information.}
    We define a predictable commitment rule for controlled systems.  The same
    rule determines the safety allowance under the alternative, the
    noncommitment regret profile under the target, and the stopped information.

    \item \textbf{Causal reduction.} We reduce safe oracle attainability to three quantities induced by the same commitment rule: the alternative-side commitment allowance, the target-side noncommitment regret profile, and the information in the learner-visible record available at the decision. This gives the exact binary-KL bound in Theorem~\ref{thm:binary-precommitment-limit}. We extend the result to
    continuous classes and use a list-Fano bound to quantify partial
    specialization among multiple alternatives.

    \item \textbf{Examples and recovery.}
    We prove linear regret in a constrained linear system with quadratic
    regulation cost and an unknown margin-consuming input direction.  An
    unknown-disturbance-location problem yields a certified safe candidate
    list.  A separate alignment condition gives oracle recovery at rate
    $O(\sqrt{T\log T})$ when safe experiments are repeatable and informative.

    \item \textbf{Computation.}
    For deterministic linear-Gaussian systems, we formulate finite-horizon
    precommitment information as a QCQP and derive SDP upper certificates.
    Numerical studies show how replenishment, sensing, and unknown-location
    structure change the attainable level of adaptation.
\end{enumerate}

%% file: safe_v15_package/setup_v14_edited_direct.tex
\section{Setup}
\label{sec:setup-general}

We define the controlled interaction, uniform safety, the parameter-aware oracle, and commitment.

\subsection{General controlled interaction}
\label{subsec:controlled-interaction-general}

Fix a horizon $T$ and an uncertainty set $\Theta_0\subseteq\Theta$.  A
state-space representation, when useful, is
\begin{equation}
    X_{t+1}=f_{\theta,t}(X_t,U_t,W_{t+1}),
    \qquad
    Z_t=h_{\theta,t}(X_t,U_t,V_t),
    \qquad t=1,\ldots,T,
    \label{eq:state-space-model}
\end{equation}
where $X_t$ is the physical state, $U_t$ is the chosen control or sensing
action, $Z_t$ is the learner-visible datum, and $(W_t,V_t)$ are noises. Uppercase denotes random variables. The
learner sees
\begin{equation*}
    H_{t-1}:=(Z_1,U_1,\ldots,Z_{t-1},U_{t-1})
\end{equation*}
before choosing $U_t\sim\pi_t(\cdot\mid H_{t-1})$.  Equivalently, under
$\theta$, the next learner-visible datum has conditional law
\begin{equation}
    Z_t\mid(H_{t-1},U_t)
    \sim Q_{\theta,t}(\cdot\mid H_{t-1},U_t).
    \label{eq:general-controlled-kernel}
\end{equation}
The kernels $Q_{\theta,t}$ are induced by the dynamics and observation model.
The lower bound uses only these controlled observation laws; it does not require
linearity, Gaussian noise, or Markov representation in
\eqref{eq:state-space-model}.

Let $\mathsf{Safe}_{\theta,T}$ be the event that all state, input, resource,
and operational constraints hold through horizon $T$ under $\theta$.  For
state constraints, it includes the post-action state $X_{T+1}$ generated by
$U_T$.  Let $\overline{P}_{\theta,T}^\pi$ denote the full physical trajectory
law, including latent safety variables, and let $P_{\theta,T}^\pi$ denote the
learner-visible marginal.  Safety probabilities are taken under
$\overline{P}_{\theta,T}^\pi$; KL divergences are taken on learner-visible
records.

\begin{definition}[Uniform chance safety]
\label{def:general-uniform-safety}
For a risk level $\eta_T\in[0,1)$, a policy is \emph{chance-$\eta_T$ safe}
under $\theta$ if
\begin{equation*}
    \overline{P}_{\theta,T}^\pi(\mathsf{Safe}_{\theta,T}^c)\le \eta_T.
\end{equation*}
It is \emph{uniformly chance-safe} on $\Theta_0$ if the same causal policy
satisfies this inequality for every initially plausible environment:
\begin{equation}
    \Pi_{\mathrm{safe}}^T(\Theta_0)
    :=
    \bigcap_{\theta\in\Theta_0}
    \left\{\pi:
    \overline{P}_{\theta,T}^\pi(\mathsf{Safe}_{\theta,T}^c)\le\eta_T
    \right\}.
    \label{eq:general-safe-class}
\end{equation}
Hard safety is the special case $\eta_T=0$.  In the lower-bound examples,
$\eta_T\to0$, often $\eta_T=T^{-2}$.  When reward is bounded, the
contribution of violation paths to expected value then vanishes.
\end{definition}

\begin{remark}[Why safety is uniform]
\label{rem:why-uniform-safety}
The policy must be certified before the true environment is known.  Uniformity
therefore applies to the safety guarantee, not to the realized behavior: a
causal policy may use observations and act differently under different
models.  Every data-dependent decision inside the policy, including a decision
to commit, remains covered by the guarantee in
\eqref{eq:general-safe-class}.
\end{remark}

\begin{remark}[Uniform safety and max--min control]
\label{rem:vs-robust-control}
The class $\Pi_{\mathrm{safe}}^T(\Theta_0)$ contains general adaptive policies;
a max--min policy may also be adaptive.  The distinction in this paper is the
performance criterion: safety is required uniformly over $\Theta_0$, whereas
performance is measured by regret to the oracle for the environment that is
actually realized.
\end{remark}

Let $\mathcal R_T$ denote realized reward through horizon $T$.  We assume
that its expectation and the oracle value below are finite.  The main theorem
imposes no pathwise reward ceiling.  A finite gross bound is used only in
Lemma~\ref{lem:binary-value-verification} to evaluate the noncommitment regret
profile in deterministic examples.
The value of a policy $\pi$ in environment $\theta$ is
\begin{equation}
    J_T(\pi;\theta):=\mathbb E_\theta^\pi[\mathcal R_T].
    \label{eq:general-value}
\end{equation}
The parameter-aware safe oracle knows $\theta$ and is constrained only by
safety under that environment:
\begin{equation}
    J_T^\star(\theta)
    :=
    \sup_{\pi\in\Pi_{\mathrm{safe}}^T(\{\theta\})}J_T(\pi;\theta).
    \label{eq:general-safe-oracle}
\end{equation}
Regret is
\begin{equation}
    R_T(\pi;\theta):=J_T^\star(\theta)-J_T(\pi;\theta).
    \label{eq:general-regret}
\end{equation}
Because $\Pi_{\mathrm{safe}}^T(\Theta_0)\subseteq
\Pi_{\mathrm{safe}}^T(\{\theta\})$ for every $\theta\in\Theta_0$, the regret of
a uniformly safe policy is nonnegative.

\begin{definition}[Safe oracle attainability]
\label{def:general-oracle-attainability}
Let
\begin{equation*}
    \mathfrak R_T(\Theta_0)
    :=
    \inf_{\pi\in\Pi_{\mathrm{safe}}^T(\Theta_0)}
    \sup_{\theta\in\Theta_0}R_T(\pi;\theta).
\end{equation*}
For a positive comparison scale $a_T$, the oracle is safely attainable at
scale $a_T$ if $\mathfrak R_T(\Theta_0)=o(a_T)$.  Average-regret learnability
corresponds to $a_T=T$.
\end{definition}

The scale matters: a lower bound of order $\sqrt T$ can rule out closure of a
$\Theta(\sqrt T)$ oracle advantage without ruling out $o(T)$ average regret.

\input{safe_v15_package/setup_2.3_v3}

\section{Fundamental Pre-Commitment Limits}
\label{sec:fundamental-limit-generalized}

The lower bound uses three quantities: commitment probability under the
alternative, the regret cost of noncommitment under the target, and the
information available before commitment.  We define each quantity directly.

\subsection{Binary target versus incompatible alternative}
\label{subsec:binary-limit}
The committing action is chosen before its observation is seen, so the record on
which information is measured should contain the committing action but not its
outcome. We therefore use a stopping time. Let
$\mathcal F_t:=\sigma(H_{t-1},U_t)$
be the $\sigma$-algebra generated by the learner-visible history and the
time-$t$ action but \emph{not} the time-$t$ observation $Z_t$.  Because whether
$(H_{t-1},U_t)\in\mathcal D_t^{\theta\mid\theta'}$ is decided by $U_t$ before
$Z_t$ is seen, the commitment set is $\mathcal F_t$-measurable, and hence
$\tau^\pi$ is an $\{\mathcal F_t\}$-stopping time.

Let $\mathcal F_{\tau^\pi}$ be the associated stopped $\sigma$-algebra.  The
stopped record
\begin{equation}
    S_T^\pi
    :=
    \left(
       \tau^\pi\wedge(T+1),\;
       H_{(\tau^\pi-1)\wedge T},\;
       \widetilde U^\pi
    \right),
    \qquad
    \widetilde U^\pi:=
    \begin{cases}
      U_{\tau^\pi},&\tau^\pi\le T,\\
      \dagger,&\tau^\pi=T+1,
    \end{cases}
    \label{eq:general-stopped-record}
\end{equation}
generates $\mathcal F_{\tau^\pi}$; the placeholder $\dagger$ marks the absence of
a committing action.  Let $P_{\theta,T}^{\pi,\mathrm{pre}}$ denote the law of
$S_T^\pi$ under $\theta$, i.e.\ the restriction of the physical law to
$\mathcal F_{\tau^\pi}$.

The commitment event
$A^\pi=\{\tau^\pi\le T\}$ is determined by the stopped record.  The physical
law and the stopped law therefore assign it the same probability, written
$P_\theta^\pi(A^\pi)$.  We apply change of measure to the stopped laws and
evaluate safety under the full physical laws.

For later use, define the one-step conditional divergence
\begin{equation}
    \mathrm{kl}^{\theta,\theta'}_t
    :=
    D_{\mathrm{KL}}\!\left(
      Q_{\theta,t}(\cdot\mid H_{t-1},U_t)
      \,\middle\|\,
      Q_{\theta',t}(\cdot\mid H_{t-1},U_t)
    \right).
    \label{eq:general-one-step-kl}
\end{equation}

\begin{definition}[Pairwise pre-commitment information]
\label{def:pairwise-pre-information}
The maximum information that can be collected safely before commitment is
\begin{equation}
    I_{\mathrm{pre}}(T;\theta,\theta')
    :=
    \sup_{\pi\in\Pi_{\mathrm{safe}}^T(\{\theta,\theta'\})}
    D_{\mathrm{KL}}\!\left(
      P_{\theta,T}^{\pi,\mathrm{pre}}
      \,\middle\|\,
      P_{\theta',T}^{\pi,\mathrm{pre}}
    \right).
    \label{eq:pairwise-pre-information}
\end{equation}
Its horizon-uniform capacity is
\begin{equation}
    \beta_{\mathrm{pre}}(\theta,\theta')
    :=\sup_{T\ge1}I_{\mathrm{pre}}(T;\theta,\theta').
    \label{eq:pairwise-pre-capacity}
\end{equation}
\end{definition}

\begin{lemma}[Stopped controlled-KL identity]
\label{lem:general-stopped-kl}
Assume
$Q_{\theta,t}(\cdot\mid h,u)\ll Q_{\theta',t}(\cdot\mid h,u)$ for every
pre-commitment history-action pair.  Then, for every causal policy,
\begin{equation}
    D_{\mathrm{KL}}\!\left(
      P_{\theta,T}^{\pi,\mathrm{pre}}
      \,\middle\|\,
      P_{\theta',T}^{\pi,\mathrm{pre}}
    \right)
    =
    \mathbb E_\theta^\pi\!\left[
      \sum_{t=1}^T
      \mathbf 1\{t<\tau^\pi\}
      \mathrm{kl}^{\theta,\theta'}_t
    \right].
    \label{eq:general-stopped-kl-identity}
\end{equation}
\end{lemma}

\begin{proof}
Since $\tau^\pi$ is an $\{\mathcal F_t\}$-stopping time with
$\mathcal F_t=\sigma(H_{t-1},U_t)$, the event $\{t<\tau^\pi\}$ is
$\mathcal F_t$-measurable, hence determined by $(H_{t-1},U_t)$ before $Z_t$ is
seen.  Form the likelihood ratio of the stopped record under $\theta$ against
$\theta'$.  The policy kernels $\pi_t(\cdot\mid H_{t-1})$ are common to both
environments and cancel, including the kernel that generates the committing
action $U_{\tau^\pi}$; only the observation kernels $Q_{\theta,t}$ versus
$Q_{\theta',t}$ survive, and the stopping time contributes no additional
likelihood factor.

If $\tau^\pi=m\le T$, the stopped record contains observations only through
time $m-1$, so the log-likelihood ratio is the sum of one-step observation
contributions for $t<m$; if no commitment occurs, the sum runs through
$t=1,\ldots,T$.  In both cases the terms present are exactly those on
$\{t<\tau^\pi\}$.  Conditioning the $t$-th term on $(H_{t-1},U_t)$ gives
$\mathrm{kl}^{\theta,\theta'}_t$ by definition
\eqref{eq:general-one-step-kl}, and taking expectations under $\theta$ yields
\eqref{eq:general-stopped-kl-identity}.
\end{proof}

\paragraph{Finite-headroom information.}
After choosing $u_t$, the controller observes
$y_t=\theta u_t+\xi_t$, where $\xi_t\sim\mathcal N(0,\sigma^2)$.  The one-step
divergence is
\begin{equation}
    \mathrm{kl}^{\theta_b,\theta_d}_t
    =D_{\mathrm{KL}}\!\left(
      \mathcal N(\theta_bu_t,\sigma^2)
      \,\middle\|\,
      \mathcal N(\theta_du_t,\sigma^2)
    \right)
    =\frac{(\theta_d-\theta_b)^2u_t^2}{2\sigma^2}.
    \label{eq:running-thermal-one-step-kl}
\end{equation}
Before commitment,
$\sum_{t<\tau^\pi}u_t^2\le B_d=d_0/\theta_d$.  Hence
Lemma~\ref{lem:general-stopped-kl} gives
\begin{equation}
    I_{\mathrm{pre}}(T;\theta_b,\theta_d)
    \le
    \beta_{\mathrm{th}}
    :=\frac{(\theta_d-\theta_b)^2d_0}{2\sigma^2\theta_d}.
    \label{eq:running-thermal-pre-information}
\end{equation}
This bound does not depend on $T$.

\begin{definition}[Commitment allowance]
\label{ass:binary-value-risk}
For an ordered pair $(\theta,\theta')$ with a nonempty uniformly safe
class, define
\begin{equation}
    \alpha_T(\theta,\theta')
    :=
    \sup_{\pi\in\Pi_{\mathrm{safe}}^T(\{\theta,\theta'\})}
    P_{\theta'}^\pi(A^\pi).
    \label{eq:binary-risk-separation}
\end{equation}
We write $\alpha_T$ when the pair is clear.  Any upper bound
$\bar\alpha_T\ge\alpha_T$ can replace it in the bounds below.  In particular,
if
\begin{equation*}
    A^\pi\subseteq \mathsf{Safe}_{\theta',T}^c
    \quad \overline P_{\theta',T}^\pi\text{-a.s.}
\end{equation*}
for every uniformly safe policy, then $\alpha_T\le\eta_T$.
\end{definition}

\paragraph{Finite-headroom device.}
Crossing the $\theta_d$-safe budget gives
\begin{equation*}
    d_{\tau^\pi+1}
    =d_0-\theta_d\sum_{s\le\tau^\pi}u_s^2<0
\end{equation*}
under $\theta_d$.  Hence commitment implies a safety violation and
$\alpha_T(\theta_b,\theta_d)\le\eta_T$.

\subsection{The noncommitment regret profile}
\label{subsec:value-gap-derived}

The stopped information bound gives a lower bound on the probability of not
committing under the target.  The next definition records the exact regret
cost of that probability.

\begin{definition}[Noncommitment regret profile]
\label{def:noncommitment-regret-profile}
\label{def:value-separation}
For an ordered pair $(\theta,\theta')$ and $q\in[0,1]$, define
\begin{equation}
    \mathcal G_T(\theta,\theta';q)
    :=
    \inf_{\substack{
      \mu\in\Pi_{\mathrm{safe}}^T(\{\theta,\theta'\})\\
      P_\theta^\mu((A^\mu)^c)\ge q
    }}
    R_T(\mu;\theta),
    \label{eq:noncommitment-regret-profile}
\end{equation}
with the convention $\inf\varnothing=+\infty$.  Equivalently, whenever the
constraint set is nonempty,
\begin{equation}
    \mathcal G_T(\theta,\theta';q)
    =J_T^\star(\theta)
    -
    \sup_{\substack{
      \mu\in\Pi_{\mathrm{safe}}^T(\{\theta,\theta'\})\\
      P_\theta^\mu((A^\mu)^c)\ge q
    }}
    J_T(\mu;\theta).
    \label{eq:noncommitment-value-profile}
\end{equation}
The function $q\mapsto\mathcal G_T(\theta,\theta';q)$ is nonnegative and
nondecreasing.
\end{definition}

This definition imposes no pathwise reward ceiling.  In a stochastic-reward
problem, the profile may be bounded directly in expectation.  The following
lemma gives a convenient pathwise bound for the deterministic examples.

\begin{lemma}[Pathwise bound on the profile]
\label{lem:binary-value-verification}
Suppose there are constants
\begin{equation*}
    0\le V_T(\theta,\theta')\le B_T(\theta)\le R_{\max,T}<\infty
\end{equation*}
such that every $\mu\in\Pi_{\mathrm{safe}}^T(\{\theta,\theta'\})$ satisfies,
$\overline P_{\theta,T}^\mu$-almost surely,
\begin{equation*}
\mathcal R_T\le
\begin{cases}
V_T(\theta,\theta'),
    &\text{on }(A^\mu)^c\cap\mathsf{Safe}_{\theta,T},\\
B_T(\theta),
    &\text{on }A^\mu\cap\mathsf{Safe}_{\theta,T},\\
R_{\max,T},
    &\text{always.}
\end{cases}
\end{equation*}
Then, for every $q\in[0,1]$,
\begin{align}
    \mathcal G_T(\theta,\theta';q)
    \ge\Bigl[&J_T^\star(\theta)-B_T(\theta)
      +\bigl(B_T(\theta)-V_T(\theta,\theta')\bigr)q \notag\\
      &-\bigl(R_{\max,T}-V_T(\theta,\theta')\bigr)\eta_T
    \Bigr]_+.
    \label{eq:binary-pathwise-value-gap}
\end{align}
\end{lemma}

\begin{proof}
Fix a feasible policy $\mu$ in
\eqref{eq:noncommitment-regret-profile}, and write
\begin{equation*}
    p:=P_\theta^\mu((A^\mu)^c),
    \qquad
    f:=\overline P_{\theta,T}^\mu(\mathsf{Safe}_{\theta,T}^c).
\end{equation*}
Let
$x:=\overline P_{\theta,T}^\mu((A^\mu)^c\cap\mathsf{Safe}_{\theta,T})$.
Then $x\ge p-f$.  The pathwise bounds give
\begin{align*}
    J_T(\mu;\theta)
    &\le
      V_Tx+B_T(1-f-x)+R_{\max,T}f\\
    &\le
      B_T-(B_T-V_T)p+(R_{\max,T}-V_T)f\\
    &\le
      B_T-(B_T-V_T)q+(R_{\max,T}-V_T)\eta_T,
\end{align*}
where the arguments $(\theta,\theta')$ have been suppressed.  Subtracting
from $J_T^\star(\theta)$ and taking the infimum over feasible policies gives
\eqref{eq:binary-pathwise-value-gap}.
\end{proof}

A useful simpler bound is
\begin{equation}
    \mathcal G_T(\theta,\theta';q)
    \ge
    \left[G_T(\theta,\theta')q-\zeta_T(\theta,\theta')\right]_+,
    \qquad q\in[0,1].
    \label{eq:binary-value-separation}
\end{equation}
The quantities $G_T$ and $\zeta_T$ are derived from the system; they are not
assumptions of the main theorem.

\begin{corollary}[Pathwise verification in deterministic systems]
\label{cor:additive-value-separation}
Suppose $J_{T,\mathrm{hard}}^\star(\theta)$ is achieved by a hard-safe policy
and bounds the reward of every hard-safe trajectory.  Suppose also that every
safe noncommitting trajectory earns at most
$V_{T,\mathrm{com}}(\theta,\theta')$ and that
$0\le\mathcal R_T\le R_{\max,T}$.  Then
\begin{align}
    \mathcal G_T(\theta,\theta';q)
    \ge
    \Bigl[&J_T^\star(\theta)-J_{T,\mathrm{hard}}^\star(\theta)
    +G_T(\theta,\theta')q \notag\\
    &-\bigl(R_{\max,T}-V_{T,\mathrm{com}}(\theta,\theta')\bigr)\eta_T
    \Bigr]_+,
    \label{eq:deterministic-profile-bound}
\end{align}
where
\begin{equation}
    G_T(\theta,\theta')
    :=J_{T,\mathrm{hard}}^\star(\theta)
      -V_{T,\mathrm{com}}(\theta,\theta').
    \label{eq:deterministic-commitment-gap}
\end{equation}
Since every hard-safe policy is chance-safe,
$J_T^\star(\theta)\ge J_{T,\mathrm{hard}}^\star(\theta)$.  Thus
\eqref{eq:binary-value-separation} holds with this $G_T$ and
$\zeta_T=R_{\max,T}\eta_T$.
\end{corollary}
These conditions hold in the deterministic examples in this paper.  For
stochastic rewards, the main theorem still applies; one bounds
$\mathcal G_T(\theta,\theta';q)$ directly in expectation.
\begin{remark}[Intuition]
\label{rem:value-quantities}
The pathwise bounds involve four values, ordered as
\begin{equation*}
    V_{T,\mathrm{com}}(\theta,\theta')
    \;\le\;
    J^\star_{T,\mathrm{hard}}(\theta)
    \;\le\;
    J^\star_T(\theta)
    \;\le\;
    R_{\max,T}.
\end{equation*}
$R_{\max,T}$ is the physical ceiling: the most any trajectory can earn,
safe or not.  $J^\star_T(\theta)$ is the benchmark: the best
chance-$\eta_T$-safe value with $\theta$ known.
$J^\star_{T,\mathrm{hard}}(\theta)$ is the best value with $\theta$ known
and no violation tolerance; it enters for a pathwise reason: on
$\mathsf{Safe}_{\theta,T}$ every realized trajectory is hard-safe, so in a
deterministic-reward system its reward cannot exceed
$J^\star_{T,\mathrm{hard}}(\theta)$, whereas the chance-safe value is an
expectation and bounds no single path.  Using
$J^\star_{T,\mathrm{hard}}\le J^\star_T$ only makes the regret bound
conservative.  $V_{T,\mathrm{com}}(\theta,\theta')$ is the recoverable
value: the most a trajectory can earn under $\theta$ while retaining a
hard-safe continuation under $\theta'$, that is, while noncommitting.
The gap $G_T=J^\star_{T,\mathrm{hard}}-V_{T,\mathrm{com}}$ is the portion
of the target's value reachable only by committing. The linear profile
bound \eqref{eq:binary-value-separation} reads: forcing noncommitment
probability $q$ forfeits at least a $q$-fraction of this gap.
On the finite-headroom device, $R_{\max,T}=UT$ (full command every round),
$J^\star_{T,\mathrm{hard}}(\theta_b)=\min\{UT,\sqrt{TB_b}\}$ (spend the
true budget), $V_{T,\mathrm{com}}=\min\{UT,\sqrt{TB_d}\}$ (spend only the
worn budget); $G_T$ is the reward available only beyond the worn budget.
\end{remark}
\paragraph{Finite-headroom device.}
For the target $\theta_b$ and alternative $\theta_d$,
\begin{equation*}
    J_{T,\mathrm{hard}}^\star(\vartheta)
    =\min\{UT,\sqrt{TB_\vartheta}\},
    \qquad B_0=+\infty,
\end{equation*}
and every safe noncommitting trajectory earns at most
\begin{equation*}
    V_{T,\mathrm{com}}(\theta_b,\theta_d)
    =\min\{UT,\sqrt{TB_d}\}.
\end{equation*}
Thus
\begin{equation}
    G_T(\theta_b,\theta_d)
    =
    \min\{UT,\sqrt{TB_b}\}
    -
    \min\{UT,\sqrt{TB_d}\}.
    \label{eq:running-thermal-gap}
\end{equation}
Since $R_{\max,T}=UT$, Corollary~\ref{cor:additive-value-separation} gives
\begin{equation}
    \mathcal G_T(\theta_b,\theta_d;q)
    \ge
    \left[G_T(\theta_b,\theta_d)q-UT\eta_T\right]_+.
    \label{eq:thermal-linear-profile}
\end{equation}

\subsection{The binary limit}
\label{subsec:binary-limit-theorem}

For $p,q\in[0,1]$, let
\begin{equation*}
    d_{\mathrm{bin}}(p\|q)
    :=p\log\frac{p}{q}+(1-p)\log\frac{1-p}{1-q},
\end{equation*}
with the usual extended-value conventions.  For $\alpha\in[0,1)$ and
$D\in[0,\infty]$, define
\begin{equation}
    \psi_\alpha(D)
    :=
    \inf\left\{
      p\in[0,1-\alpha]:
      d_{\mathrm{bin}}(p\|1-\alpha)\le D
    \right\}.
    \label{eq:binary-kl-inverse}
\end{equation}
Set $\psi_1(D):=0$ and $\psi_\alpha(\infty):=0$.  The function
$\psi_\alpha$ is nonincreasing.  It is the smallest target noncommitment
probability compatible with divergence $D$ and alternative commitment
probability at most $\alpha$.

\begin{theorem}[Binary pre-commitment limit]
\label{thm:binary-precommitment-limit}
Fix an ordered pair $(\theta,\theta')$ whose uniformly safe class is nonempty.
For
$\pi\in\Pi_{\mathrm{safe}}^T(\{\theta,\theta'\})$, define
\begin{equation*}
    D_T^\pi
    :=D_{\mathrm{KL}}\!\left(
      P_{\theta,T}^{\pi,\mathrm{pre}}
      \,\middle\|\,
      P_{\theta',T}^{\pi,\mathrm{pre}}
    \right),
    \qquad
    a_T^\pi:=P_{\theta'}^\pi(A^\pi).
\end{equation*}
Then
\begin{equation}
    R_T(\pi;\theta)
    \ge
    \mathcal G_T\!\left(
      \theta,\theta';
      \psi_{a_T^\pi}(D_T^\pi)
    \right).
    \label{eq:binary-policy-exact-bound}
\end{equation}
Consequently,
\begin{equation}
    R_T(\pi;\theta)
    \ge
    \mathcal G_T\!\left(
      \theta,\theta';
      \psi_{\alpha_T(\theta,\theta')}\!\left(
        I_{\mathrm{pre}}(T;\theta,\theta')
      \right)
    \right).
    \label{eq:binary-exact-profile-bound}
\end{equation}
Since
\begin{equation*}
    \psi_\alpha(D)
    \ge
    \left[\frac12e^{-D}-\alpha\right]_+,
\end{equation*}
this implies the explicit bound
\begin{equation}
    R_T(\pi;\theta)
    \ge
    \mathcal G_T\!\left(
      \theta,\theta';
      \left[
        \frac12 e^{-I_{\mathrm{pre}}(T;\theta,\theta')}
        -\alpha_T(\theta,\theta')
      \right]_+
    \right).
    \label{eq:binary-bh-profile-bound}
\end{equation}
If the linear profile bound \eqref{eq:binary-value-separation} has been
verified, then
\begin{equation}
    R_T(\pi;\theta)
    \ge
    \left[
      G_T(\theta,\theta')\left(
        \frac12 e^{-I_{\mathrm{pre}}(T;\theta,\theta')}
        -\alpha_T(\theta,\theta')
      \right)
      -\zeta_T(\theta,\theta')
    \right]_+.
    \label{eq:binary-profile-bound}
\end{equation}
If, in addition,
$\beta_{\mathrm{pre}}(\theta,\theta')<\infty$,
$\alpha_T(\theta,\theta')\to0$, and
$\zeta_T(\theta,\theta')=o(G_T(\theta,\theta'))$, then
\begin{equation}
    \liminf_{T\to\infty}
    \inf_{\pi\in\Pi_{\mathrm{safe}}^T(\{\theta,\theta'\})}
    \frac{\sup_{\vartheta\in\{\theta,\theta'\}}R_T(\pi;\vartheta)}
         {G_T(\theta,\theta')}
    \ge
    \frac12e^{-\beta_{\mathrm{pre}}(\theta,\theta')}.
    \label{eq:binary-asymptotic-bound}
\end{equation}
\end{theorem}
\noindent \emph{Interpretation.}
Equation~\eqref{eq:binary-exact-profile-bound} is the exact causal
reduction; Equation~\eqref{eq:binary-bh-profile-bound} is a convenient
explicit relaxation, and Equation~\eqref{eq:binary-profile-bound}
follows after inserting the linear noncommitment-profile bound.
Safety controls the commitment probability under the alternative,
the stopped experiment controls how different the target commitment
probability can be, and the noncommitment profile converts the
remaining target-side probability into regret.
\begin{proof}
Let
\begin{equation*}
    p:=P_\theta^\pi((A^\pi)^c),
    \qquad
    q:=P_{\theta'}^\pi((A^\pi)^c)=1-a_T^\pi.
\end{equation*}
Data processing for the event $(A^\pi)^c$ gives
\begin{equation}
    d_{\mathrm{bin}}(p\|q)\le D_T^\pi.
    \label{eq:binary-data-processing}
\end{equation}
If $p\le q$, the definition of $\psi$ gives
$p\ge\psi_{a_T^\pi}(D_T^\pi)$; if $p>q$, the same inequality is immediate.
The policy $\pi$ is therefore feasible in the definition of the profile at
this value, which proves \eqref{eq:binary-policy-exact-bound}.

For the uniform bound, Definition~\ref{ass:binary-value-risk} gives
$q\ge1-\alpha_T(\theta,\theta')$.  If
$p\le1-\alpha_T$, then
$d_{\mathrm{bin}}(p\|1-\alpha_T)\le d_{\mathrm{bin}}(p\|q)$; if
$p>1-\alpha_T$, the same lower bound is immediate.  Hence
\begin{equation*}
    p\ge\psi_{\alpha_T(\theta,\theta')}(D_T^\pi).
\end{equation*}
Since
\begin{equation*}
    D_T^\pi\le I_{\mathrm{pre}}(T;\theta,\theta')
\end{equation*}
and $\psi_\alpha$ is nonincreasing,
\eqref{eq:binary-exact-profile-bound} follows.  Bretagnolle--Huber gives
\begin{equation*}
    \psi_\alpha(D)
    \ge
    \left[\frac12e^{-D}-\alpha\right]_+.
\end{equation*}
This yields \eqref{eq:binary-bh-profile-bound}.  Substituting
\eqref{eq:binary-value-separation} gives \eqref{eq:binary-profile-bound}.
The asymptotic statement follows from the horizon-uniform information bound
and the stated limits.
\end{proof}

The explicit term in \eqref{eq:binary-bh-profile-bound} is positive when
\begin{equation*}
    I_{\mathrm{pre}}(T;\theta,\theta')
    <\log\frac{1}{2\alpha_T(\theta,\theta')}.
\end{equation*}
For $\alpha_T=T^{-2}$, the threshold is $2\log T-\log2$.

\paragraph{Finite-headroom device.}
Using $\alpha_T(\theta_b,\theta_d)\le\eta_T$,
\eqref{eq:running-thermal-pre-information}, and
\eqref{eq:thermal-linear-profile}, Theorem~\ref{thm:binary-precommitment-limit}
gives
\begin{equation}
    R_T(\pi;\theta_b)
    \ge
    \left[
      G_T(\theta_b,\theta_d)
      \left(\frac12e^{-\beta_{\mathrm{th}}}-\eta_T\right)
      -UT\eta_T
    \right]_+.
    \label{eq:running-thermal-final-bound}
\end{equation}
With $\eta_T=T^{-2}$, fixed $0<\theta_b<\theta_d$, and an inactive per-step
cap, $G_T(\theta_b,\theta_d)=\Theta(\sqrt T)$.  At $\theta_b=0$,
$G_T(0,\theta_d)=UT-O(\sqrt T)=\Theta(T)$, so the same information bound gives
linear regret.

\begin{corollary}[Information required for a target regret]
\label{cor:binary-required-information}
Suppose the linear profile bound \eqref{eq:binary-value-separation} holds.  If a
uniformly safe policy satisfies $R_T(\pi;\theta)\le r_T$, then
\begin{equation}
    I_{\mathrm{pre}}(T;\theta,\theta')
    \ge
    \left[
      \log\frac{1}{
        2\left(
          \alpha_T(\theta,\theta')
          +\frac{r_T+\zeta_T(\theta,\theta')}
                 {G_T(\theta,\theta')}
        \right)
      }
    \right]_+.
    \label{eq:binary-required-information}
\end{equation}
Thus $r_T=o(G_T)$ requires unbounded pre-commitment information whenever
$\alpha_T\to0$ and $\zeta_T=o(G_T)$.
\end{corollary}

\begin{proof}
Apply \eqref{eq:binary-profile-bound} and rearrange.
\end{proof}

\begin{remark}[Other risk constraints]
\label{rem:cvar-commitment-ceiling}
The theorem uses safety only through the commitment allowance $\alpha_T$.  Any
risk constraint that bounds $P_{\theta'}^\pi(A^\pi)$ can be used.  For example,
if a nonnegative loss satisfies
$L_T\ge m_T\mathbf 1\{A^\pi\}$ under $\theta'$ and
$\mathrm{CVaR}_\beta(L_T)\le c_T$, then
\begin{equation*}
    P_{\theta'}^\pi(A^\pi)
    \le\frac{c_T}{m_T}.
\end{equation*}
Specifications that budget cumulative violations rather than bounding the
violation probability under every plausible model need not imply a small
commitment allowance and therefore fall outside the premise of the theorem;
in the safe linear bandit setting, \cite{GangradeChenSaligrama24} show that
such a budget purchases exactly the boundary exploration that uniform
chance safety forbids, and near-optimal learning is then possible.
\end{remark}

\begin{remark}[Generality]
\label{rem:generality}
The theorem uses only the two stopped laws and data processing.  It does not
require Markov dynamics, Gaussian noise, linearity, full-state observation, or
absolute continuity.  Absolute continuity is needed only for the sum identity
in Lemma~\ref{lem:general-stopped-kl}.  An application must define a
predictable commitment rule and bound $\alpha_T$, $I_{\mathrm{pre}}$, and
$\mathcal G_T$.
\end{remark}

\subsection{Continuous classes: optimizing over the alternative}
\label{subsec:continuous-worst-alternative}

The binary theorem applies to every incompatible pair in $\Theta_0$.  If
$\Theta_0$ is continuous, the pairwise bound can be maximized over the
alternative.

\begin{corollary}[Worst-case alternative and minimax bound]
\label{cor:continuous-sup}
Let $\Theta_0$ be arbitrary.  For a target $\theta\in\Theta_0$, let
$\mathcal A(\theta)\subseteq\Theta_0$ contain alternatives with predictable
commitment rules.  Each alternative may use its own rule.  Then every
$\pi\in\Pi_{\mathrm{safe}}^T(\Theta_0)$ satisfies
\begin{equation}
    R_T(\pi;\theta)
    \ge
    \sup_{\theta'\in\mathcal A(\theta)}
    \mathcal G_T\!\left(
      \theta,\theta';
      \psi_{\alpha_T(\theta,\theta')}\!\left(
        I_{\mathrm{pre}}(T;\theta,\theta')
      \right)
    \right),
    \label{eq:continuous-per-target-bound}
\end{equation}
and
\begin{equation}
    \mathfrak R_T(\Theta_0)
    \ge
    \sup_{\theta\in\Theta_0}
    \sup_{\theta'\in\mathcal A(\theta)}
    \mathcal G_T\!\left(
      \theta,\theta';
      \psi_{\alpha_T(\theta,\theta')}\!\left(
        I_{\mathrm{pre}}(T;\theta,\theta')
      \right)
    \right).
    \label{eq:continuous-minimax-bound}
\end{equation}
\end{corollary}

\begin{proof}
Fix $\theta'\in\mathcal A(\theta)$.  A policy safe on $\Theta_0$ is safe on
$\{\theta,\theta'\}$, so Theorem~\ref{thm:binary-precommitment-limit} applies.
Taking the supremum over alternatives gives
\eqref{eq:continuous-per-target-bound}.  Taking the supremum over targets and
then the infimum over uniformly safe policies gives
\eqref{eq:continuous-minimax-bound}.
\end{proof}

If each pair satisfies the linear profile bound
\eqref{eq:binary-value-separation}, Corollary~\ref{cor:continuous-sup} gives the
explicit lower bound
\begin{equation}
    R_T(\pi;\theta)
    \ge
    \sup_{\theta'\in\mathcal A(\theta)}
    \left[
      G_T(\theta,\theta')\left(
        \frac12e^{-I_{\mathrm{pre}}(T;\theta,\theta')}
        -\alpha_T(\theta,\theta')
      \right)
      -\zeta_T(\theta,\theta')
    \right]_+.
    \label{eq:continuous-linear-bound}
\end{equation}
The optimization balances value and distinguishability.  Nearby alternatives
are harder to distinguish but may have a smaller value gap.

\paragraph{The finite-headroom device on a continuum.}
Let $\Theta_0=[0,\bar\theta]$, with
$B_\vartheta=d_0/\vartheta$ for $\vartheta>0$ and $B_0=+\infty$.  For a target
$\theta$ and alternative $\theta'\in(\theta,\bar\theta]$, commitment is the
first exit from the $\theta'$-budget.  Then
\begin{equation*}
    \alpha_T(\theta,\theta')\le\eta_T,
    \qquad
    \zeta_T(\theta,\theta')=UT\eta_T,
\end{equation*}
\begin{equation*}
    G_T(\theta,\theta')
    =\min\{UT,\sqrt{TB_\theta}\}
     -\min\{UT,\sqrt{TB_{\theta'}}\},
\end{equation*}
and
\begin{equation*}
    I_{\mathrm{pre}}(T;\theta,\theta')
    \le
    \beta(\theta,\theta')
    :=\frac{(\theta'-\theta)^2d_0}{2\sigma^2\theta'}.
\end{equation*}
Ignoring the vanishing $\eta_T$ terms, define
\begin{equation*}
    S_T(\theta,\theta')
    :=\frac12e^{-\beta(\theta,\theta')}G_T(\theta,\theta').
\end{equation*}

\emph{Interior target.}
Let $\theta>0$ and suppose the per-step cap is inactive.  Then
\begin{equation*}
    S_T(\theta,\theta')
    =\frac12\sqrt{Td_0}\,f_\theta(\theta'),
    \qquad
    f_\theta(\theta')
    :=e^{-\beta(\theta,\theta')}
      \bigl(\theta^{-1/2}-\theta'^{-1/2}\bigr).
\end{equation*}
The function $f_\theta$ does not depend on $T$, so the optimized lower bound is
$\Theta(\sqrt T)$.

\emph{Boundary target.}
For $\theta=0$,
\begin{equation*}
    G_T(0,\theta')=UT-\sqrt{Td_0/\theta'},
    \qquad
    \beta(0,\theta')=\frac{\theta'd_0}{2\sigma^2}.
\end{equation*}
Optimizing $S_T(0,\theta')$ gives
\begin{equation}
    \theta'^\star
    =\left(\frac{\sigma^4}{U^2d_0}\right)^{1/3}
      T^{-1/3}(1+o(1)),
    \qquad
    \beta^\star
    =\frac12\left(\frac{d_0}{\sigma U}\right)^{2/3}
      T^{-1/3}(1+o(1)),
    \label{eq:boundary-worst-alternative}
\end{equation}
and
\begin{equation}
    \sup_{\theta'\in(0,\bar\theta]}S_T(0,\theta')
    =\frac{UT}{2}
    -\frac34\left(\frac{d_0^2U}{\sigma^2}\right)^{1/3}T^{2/3}
    +o(T^{2/3}).
    \label{eq:boundary-sup-value}
\end{equation}
Thus
\begin{equation}
    \mathfrak R_T([0,\bar\theta])
    \ge\frac{UT}{2}(1-o(1)).
    \label{eq:continuous-minimax-thermal}
\end{equation}
The optimizing alternative approaches the boundary target at rate $T^{-1/3}$.

\subsection{One policy, many alternatives: safe lists}
\label{subsec:list-limit}

\begin{remark}[One policy against many alternatives]
\label{rem:multi-alternative-fano}
The binary theorem is pairwise.  Its information profile may use a different
policy for each alternative.  If the learner must choose among $M$
model-specific regimes, one policy must gather information about all $M$
possibilities.  The relevant quantity is
\begin{equation}
    I_{\mathrm{pre}}^{(M)}(T;\Theta_M)
    :=
    \sup_{\pi\in\Pi_{\mathrm{safe}}^T(\Theta_M)}I_\pi(J;S_T^\pi),
    \qquad J\sim\operatorname{Unif}[M].
    \label{eq:multiclass-information}
\end{equation}
When partial identification already improves control, the target is a
certified candidate list rather than exact classification.
\end{remark}

\begin{proposition}[Information required for a safe candidate list]
\label{prop:list-fano}
Let $J\sim\operatorname{Unif}[M]$, let $S$ be the information available when a
nonempty list $C(S)\subseteq[M]$ is selected, and fix
$L\in\{1,\ldots,M-1\}$.  Define
\begin{equation}
    \varepsilon_L
    :=P\bigl(J\notin C(S)\ \text{or}\ |C(S)|>L\bigr).
    \label{eq:list-fano-error}
\end{equation}
Then
\begin{equation}
    I(J;S)
    \ge
    (1-\varepsilon_L)\log\frac{M}{L}-h_2(\varepsilon_L),
    \label{eq:list-fano-bound}
\end{equation}
where $h_2(p):=-p\log p-(1-p)\log(1-p)$.  Consequently, if
\begin{equation*}
    P(J\notin C(S))\le\eta,
    \qquad
    P(|C(S)|>L)\le\delta,
    \qquad
    \eta+\delta\le\tfrac12,
\end{equation*}
then
\begin{equation}
    I(J;S)
    \ge
    (1-\eta-\delta)\log\frac{M}{L}-h_2(\eta+\delta).
    \label{eq:list-fano-two-sided}
\end{equation}
\end{proposition}

The proof is in the appendix.  The case $L=M$ gives no reduction, $L=1$ gives
exact specialization, and $1<L<M$ gives partial specialization.

\begin{remark}[Scalar continuous candidate sets]
\label{rem:list-rate-distortion}
Let $J\sim\operatorname{Unif}[\theta_{\min},\theta_{\max}]$, with
$D=\theta_{\max}-\theta_{\min}$.  Suppose the learner returns a set $C(S)$
such that
\begin{equation*}
    P\bigl(J\notin C(S)\ \text{or}\ \operatorname{diam}C(S)>\ell\bigr)
    \le\varepsilon\le\tfrac12.
\end{equation*}
Then
\begin{equation}
    I(J;S)
    \ge
    (1-\varepsilon)\log\frac{D}{\ell}-h_2(\varepsilon),
    \qquad
    \ell
    \ge
    D\exp\!\left(
      -\frac{I(J;S)+h_2(\varepsilon)}{1-\varepsilon}
    \right).
    \label{eq:continuous-list-fano}
\end{equation}
This is a converse for a scalar parameter on an interval.  Achieving the bound
still requires a safe sensing policy.
\end{remark}

\subsection{How to use the limits in an application}
\label{subsec:limit-interpretation}

For a new system:
\begin{enumerate}
    \item define a predictable commitment event;
    \item bound the commitment allowance $\alpha_T$ under the incompatible
    alternative;
    \item compute or lower-bound the noncommitment regret profile
    $\mathcal G_T$; and
    \item bound the stopped KL divergence, or the stopped mutual information in
    a multi-alternative problem.
\end{enumerate}
The theorem then gives a regret lower bound.  A positive result requires a
single uniformly safe policy that gathers the needed information and moves
safely to the selected regime.

%% file: safe_v15_package/setup_2.3_v3.tex
\subsection{Commitment and information before commitment}
\label{subsec:commitment-general}

Fix an ordered pair $(\theta,\theta')\in\Theta_0^2$.  The first environment is
the target: leaving the common regime may be valuable there.  The second is an
incompatible alternative: the same decision is not safely recoverable there.
These roles are relative to the decision under study; they do not impose a
global ordering on the model class.

In the deterministic constrained systems used below, let
$\mathcal U_t^{\mathrm{viab}}(h;\vartheta)$ denote the actions available after
history $h$ for which at least one continuation remains hard-safe under
$\vartheta$ through the horizon.

\begin{definition}[Pairwise commitment]
\label{def:pairwise-commitment}
For the ordered pair $(\theta,\theta')$, define
\begin{equation}
    \mathcal D_t^{\theta\mid\theta'}
    :=\bigl\{(h,u):u\notin
    \mathcal U_t^{\mathrm{viab}}(h;\theta')\bigr\}.
    \label{eq:viability-induced-commitment}
\end{equation}
For a policy $\pi$, let
\begin{equation}
    \tau^\pi
    :=\inf\bigl\{t\le T:(H_{t-1},U_t)
    \in\mathcal D_t^{\theta\mid\theta'}\bigr\},
    \label{eq:pairwise-commitment-time}
\end{equation}
with $\tau^\pi=T+1$ if the set is empty, and define
\begin{equation}
    A^\pi:=\{\tau^\pi\le T\}.
    \label{eq:pairwise-commitment-event}
\end{equation}
\end{definition}

Thus commitment is the first action after which no hard-safe continuation
remains under the incompatible alternative.  Because membership in
$\mathcal D_t^{\theta\mid\theta'}$ is determined by $(H_{t-1},U_t)$, the
commitment decision is made before the observation $Z_t$ generated by that
action is available.  For a general stochastic model, the same stopping-time
construction may be used with an application-specific predictable set
$\mathcal D_t^{\theta\mid\theta'}$; its safety consequence is then quantified by bounding the commitment allowance of Definition~\ref{ass:binary-value-risk} below.  No general
chance-constrained viability operator is needed for the lower bound.

\paragraph{Finite-headroom commitment.} Recall from Section~\ref{sec:intro} the controller chooses $u_t\in[0,U]$, receives reward $r_t=u_t$, and the safety margin evolves as
$     d_{t+1}=d_t-\theta u_t^2, \,\,\, d_1=d_0>0.$ For an ordered pair $0\le\theta_b<\theta_d$, write $B_d:=B_{\theta_d}=d_0/\theta_d$. Let
\begin{equation*}
    E_{t-1}(h):=\sum_{s=1}^{t-1}u_s^2.
\end{equation*}
Since future commands can be set to zero, for $\vartheta>0$,
\begin{equation}
    \mathcal U_t^{\mathrm{viab}}(h;\vartheta)
    =\left[0,\min\left\{U,
      \sqrt{\left(\frac{d_0}{\vartheta}-E_{t-1}(h)\right)_+}
    \right\}\right].
    \label{eq:thermal-viable-action-set}
\end{equation}
For the ordered pair $(\theta_b,\theta_d)$, commitment is therefore the first
exit from the dangerous-device budget:
\begin{equation}
    \mathcal D_t^{\theta_b\mid\theta_d}
    =\bigl\{(h,u):E_{t-1}(h)+u^2>B_d\bigr\}.
    \label{eq:thermal-commitment-set}
\end{equation}

%% file: safe_v15_package/mechanisms_v5_edited.tex
\section{Mechanisms and Examples}
\label{sec:mechanisms-examples}

Section~\ref{subsec:limit-interpretation} gives the application procedure for
the lower bound.  Here we study two mechanisms that bound the information
available before commitment: finite expenditure and a finite sensing window.
We apply them to a constrained linear-regulation problem and an
unknown-location problem.  The first yields linear regret; the second yields
safe partial adaptation through a certified candidate list.

\subsection{Finite expenditure before commitment}
\label{subsec:finite-expenditure}

\begin{lemma}[Finite pre-commitment expenditure]
\label{lem:finite-expenditure}
Fix an ordered pair $(\theta,\theta')$ and its commitment time $\tau^\pi$.
Suppose there is a nonnegative predictable quantity
$S_t=S(H_{t-1},U_t)$, constants $\kappa<\infty$ and $q\ge1$, and a number
$B_T<\infty$ such that every
$\pi\in\Pi_{\mathrm{safe}}^T(\{\theta,\theta'\})$ satisfies, pathwise,
\begin{equation}
    \sum_{t=1}^T \mathbf 1\{t<\tau^\pi\}S_t\le B_T,
    \label{eq:finite-expenditure-budget}
\end{equation}
and, whenever $t<\tau^\pi$,
\begin{equation}
    \mathrm{kl}^{\theta,\theta'}_t\le \kappa S_t^q.
    \label{eq:finite-expenditure-local-kl}
\end{equation}
Then
\begin{equation}
    I_{\mathrm{pre}}(T;\theta,\theta')\le \kappa B_T^q.
    \label{eq:finite-expenditure-information}
\end{equation}
In particular, if $\sup_T B_T<\infty$, then
$\beta_{\mathrm{pre}}(\theta,\theta')<\infty$.
\end{lemma}

The proof is in Appendix~\ref{app:proof-finite-expenditure}.  Since
$\sum_{t<\tau^\pi} S_t^q\le(\sum_{t<\tau^\pi} S_t)^q$ for $q\ge1$, the
pathwise expenditure bound directly bounds the stopped information.  For
$q=1$, information is at most proportional to expenditure.  For $q>1$,
splitting a finite budget among many small actions cannot produce persistent
information.

\begin{remark}[One-sided versus Signed Actions]
\label{rem:signed-parity}
The conversion
\[
    \sum_t u_t\le B
    \quad\Longrightarrow\quad
    \sum_t u_t^2\le\min\{B^2,UB\}
\]
requires $u_t\ge0$.  With signed inputs and only a constraint on the
signed sum, cancellations can keep $\sum_t u_t$ bounded while
$\sum_t u_t^2$ grows.  This caveat is specific to the linear
one-sided margin of Variant~I.  In Variant~II, safety directly bounds
the signed expenditure
\(
    \sum_t(a^\top u_t)^2,
\)
so reversing the active loading consumes the same quadratic headroom
rather than restoring it.  The information ceiling therefore persists
under fully signed requests and controls.
\end{remark}
\input{safe_v15_package/lq_two_variants_compact_edited}
\subsection{Finite sensing bandwidth before an evolving deadline}
\label{subsec:finite-sensing-window}

The linear-regulation example limits information because informative control
spends a hidden safety margin.  A different obstruction occurs when the
state evolves toward a boundary while the learner must search across many
possible sensing locations.

\begin{lemma}[Finite pre-commitment sensing window]
\label{lem:finite-sensing-window}
Consider a labeled problem with $J\sim\operatorname{Unif}[M]$.  Suppose that,
under every policy, at most $N_T$ label-dependent observations can occur
before specialization or an absorbing fallback, and each such observation
satisfies
\begin{equation}
    I_\pi(J;Z_t\mid H_{t-1},U_t,t<\tau^\pi)
    \le c_{\mathrm{obs}}.
    \label{eq:sensing-window-one-step}
\end{equation}
Then
\begin{equation}
    I_{\mathrm{pre}}^{(M)}(T;\Theta_M)
    \le N_T c_{\mathrm{obs}}.
    \label{eq:sensing-window-total}
\end{equation}
For the one-location Gaussian sensor
\begin{equation}
    Z_t=\mu\mathbf 1\{A_t=J\}+\xi_t,
    \qquad
    \xi_t\sim\mathcal N(0,\sigma^2),
    \label{eq:sensing-window-gaussian-sensor}
\end{equation}
where $A_t\in[M]$ is chosen adaptively, one may take
\begin{equation}
    c_{\mathrm{obs}}
    =
    \min\left\{\log 2,\frac{\mu^2}{2\sigma^2}\right\}.
    \label{eq:sensing-window-cobs}
\end{equation}
\end{lemma}

The proof is in Appendix~\ref{app:proof-sensing-window}.  This is the
one-policy search effect identified in Remark~\ref{rem:multi-alternative-fano}:
one policy must search over all labels, while each round can inspect only one
location. The same bandit-feedback search structure governs quickest changepoint
detection with adaptive sensing \cite{GopalanLS21}, where the delay--false-alarm
tradeoff arises.

\subsection{Unknown disturbance location: safe lists and partial adaptation}
\label{subsec:unknown-location-lists}

There are $M$ components and one unknown affected location $J\in[M]$.
Fix an integer $1\le N<T$.  During diagnostic operation,
\begin{equation}
    X_{i,t+1}=X_{i,t}+\Delta\mathbf 1\{i=J\},
    \qquad
    X_{i,1}=0,
    \qquad i\in[M],
    \label{eq:location-diagnostic-state}
\end{equation}
and safety requires $X_{i,t}\le H$ for all $i,t$.  Write
\begin{equation}
    H=N\Delta.
    \label{eq:location-deadline}
\end{equation}
Thus full diagnostic operation is safe for $N$ rounds; on the next round the
affected component would cross the boundary unless mitigation is selected.
At each diagnostic round the learner chooses one sensor $A_t\in[M]$ and
observes
\begin{equation}
    Z_t=\mu\mathbf 1\{A_t=J\}+\xi_t,
    \qquad
    \xi_t\stackrel{\mathrm{iid}}{\sim}\mathcal N(0,\sigma^2).
    \label{eq:location-sensor}
\end{equation}
Diagnostic operation earns unit reward per round.  At the deadline the learner
selects a nonempty certified list $\mathcal C_N\subseteq[M]$.  The
post-diagnostic controller divides both throughput and mitigation across the
retained candidates:
\begin{equation}
    q(\mathcal C_N)=\frac1{|\mathcal C_N|},
    \qquad
    u_i(\mathcal C_N)
    =
    \frac{\Delta}{|\mathcal C_N|}\mathbf 1\{i\in\mathcal C_N\}.
    \label{eq:location-list-controller}
\end{equation}
After the list is chosen,
\begin{equation}
    X_{i,t+1}
    =
    X_{i,t}
    +\Delta q(\mathcal C_N)\mathbf 1\{i=J\}
    -u_i(\mathcal C_N).
    \label{eq:location-post-list-state}
\end{equation}
If $J\in\mathcal C_N$, the affected state is arrested.  If
$J\notin\mathcal C_N$, it violates the constraint on the next round; we treat
this as an absorbing zero-reward failure.  Thus safety of the list policy is
equivalent to retaining the true location.

Proposition~\ref{prop:list-fano} and
Lemma~\ref{lem:finite-sensing-window} give the following list-size bound.
If a chance-safe policy returns a list
$\mathcal C_N$ with
$\mathbb P(J\notin\mathcal C_N)\le\eta$ and
$\mathbb P(|\mathcal C_N|>L)\le\delta$, then necessarily
\begin{equation}
    N\,c_{\mathrm{sens}}
    \ge
    (1-\eta-\delta)\log\frac{M}{L}
    -h_2(\eta+\delta),
    \qquad
    c_{\mathrm{sens}}
    :=
    \min\left\{\log 2,\frac{\mu^2}{2\sigma^2}\right\}.
    \label{eq:location-list-necessary-information}
\end{equation}
Equivalently, for a deterministic size cap and failure probability at most
$\eta\le1/2$,
\begin{equation}
    L
    \ge
    M\exp\left\{
      -\frac{N c_{\mathrm{sens}}+h_2(\eta)}{1-\eta}
    \right\}.
    \label{eq:location-list-size-lower-bound}
\end{equation}
Finite pre-commitment information can therefore reduce the uncertainty set
even when it is insufficient for exact localization.

One uniformly safe elimination rule is the following.  For candidate $j$, let
\begin{equation}
    \Lambda_{j,t}
    :=
    \sum_{s\le t:A_s=j}
    \log\frac{\phi_{0,\sigma}(Z_s)}{\phi_{\mu,\sigma}(Z_s)},
    \label{eq:location-elimination-statistic}
\end{equation}
where $\phi_{m,\sigma}$ is the density of $\mathcal N(m,\sigma^2)$, and set
\begin{equation}
    \mathcal C_t
    :=
    \left\{j:\Lambda_{j,t}<\log\frac1{\eta_T}\right\}.
    \label{eq:location-safe-list-rule}
\end{equation}
Under the true location, $\exp(\Lambda_{J,t})$ is a likelihood-ratio
martingale, so Ville's inequality gives
\begin{equation}
    \mathbb P_J\!\bigl(J\notin\mathcal C_t\text{ for some }t\le N\bigr)
    \le \eta_T.
    \label{eq:location-list-safety-guarantee}
\end{equation}
If the set in \eqref{eq:location-safe-list-rule} is empty, the learner
returns $[M]$.  Otherwise, it can inspect, for example, the least-sampled
surviving location.  This rule is not claimed to be optimal; it only shows
that partial certification can be uniformly safe.

Under the reward convention above, the robust, list-adaptive, and oracle
values are
\begin{align}
    J_T^{\mathrm{rob}}
    &=
    N+\frac{T-N}{M},
    \label{eq:location-robust-value}
    \\
    J_T^{\mathrm{list}}
    &=
    N+(T-N)
    \mathbb E\!
    \left[
      \frac{\mathbf 1\{J\in\mathcal C_N\}}{|\mathcal C_N|}
    \right],
    \label{eq:location-list-value}
    \\
    J_T^{\mathrm{oracle}}
    &=T.
    \label{eq:location-oracle-value}
\end{align}
The list policy improves on robust control when
\begin{equation*}
    \mathbb E\!\left[
      \frac{\mathbf 1\{J\in\mathcal C_N\}}{|\mathcal C_N|}
    \right]
    >\frac1M,
\end{equation*}
and it reaches the oracle only when $\mathcal C_N=\{J\}$ almost surely.
Whenever the first inequality is strict and exact localization does not hold,
\begin{equation}
    J_T^{\mathrm{rob}}
    <
    J_T^{\mathrm{list}}
    <
    J_T^{\mathrm{oracle}}.
    \label{eq:location-partial-ordering}
\end{equation}
Section~\ref{subsec:numerics-unknown-location} evaluates this rule over a
range of sensing windows.

\begin{table}[t]
\centering
\small
\renewcommand{\arraystretch}{1.18}
\begin{tabular}{@{}p{0.24\linewidth}p{0.27\linewidth}p{0.15\linewidth}p{0.23\linewidth}@{}}
\hline
\textbf{Example} & \textbf{Pre-commitment information} &
\textbf{Oracle gap} & \textbf{Consequence} \\
\hline
Unknown active constraint direction (linear regulation) &
bounded by finite margin expenditure. &
$\Theta(T)$ & linear regret under uniform safety \\
Unknown disturbance location &
at most $N c_{\mathrm{sens}}$ before mitigation &
list dependent & certified robust, partial, or exact adaptation \\
Repeatable safe experiment &
can grow linearly &
model dependent & oracle recovery when the transition remains safe \\
\hline
\end{tabular}
\caption{The effect
of bounded information depends on the oracle gap.  In a multi-alternative
problem, finite information can instead determine the size of a certified
safe list.}
\label{tab:mechanism-taxonomy}
\end{table}

%% file: safe_v15_package/lq_two_variants_compact_edited.tex
\subsection{Constrained linear regulation under two safety geometries}
\label{sec:linear_regulation_unknown_security}
\label{sec:lq-two-geometries}

We use a common two-channel output geometry to exhibit two finite-expenditure
mechanisms in linear--quadratic regulation.  Variant~I uses nonnegative reserve
deployment and a linear security margin; its information ceiling follows from
an order mismatch between linear expenditure and quadratic Gaussian
information.  Variant~II permits genuinely signed regulation and uses a
quadratic $I^2t$-type margin.  The second variant shows that the obstruction is
not an artifact of one-sided actuation.

\paragraph{Shared observation geometry.}
Let
\begin{equation}
    a\in\{a_0,a_1\},
    \qquad
    a_0=e_2,
    \qquad
    a_1=e_1,
    \label{eq:lq-shared-pair}
\end{equation}
and write $u_t=(p_t,q_t)$.  The regulated state $x_t$ is observed before the
action.  After applying $u_t$, the controller observes
\begin{equation}
    y_t=a^\top u_t+\xi_t,
    \qquad
    \xi_t\stackrel{\mathrm{iid}}{\sim}\mathcal N(0,\sigma^2).
    \label{eq:lq-shared-output}
\end{equation}
The active direction and the physical margin are not directly observed.  The
initial margin and its update law are known, so the controller can propagate
the candidate margin associated with each direction from its action history.
In both variants the target is $a_0$: the $p$-channel is then invisible to the
active security constraint, whereas under $a_1$ that same channel spends the
margin.

In each variant the controller minimizes a quadratic cost and is compared with
the parameter-aware chance-safe oracle.  Stable dynamics and bounded inputs
give a uniform stage-cost ceiling.  Subtracting the stage cost from this
ceiling converts the example to the bounded-reward convention without changing
regret, with an $O(T\eta_T)$ value-profile remainder.

\paragraph{Variant I: one-sided reserve deployment.}
Let
\begin{equation}
    u_t=(p_t,q_t)\in[0,U]^2,
    \qquad
    x_{t+1}=\rho x_t+w-p_t,
    \qquad
    x_1=\bar x:=\frac{w}{1-\rho},
    \label{eq:lq-one-sided-plant}
\end{equation}
where $0<\rho<1$ and $0<w\le U$, and define
\begin{equation}
    C_T^{(1)}(\pi;a)
    :=
    \mathbb E_a^\pi
    \sum_{t=1}^T
    \left[x_{t+1}^2+\lambda(p_t^2+q_t^2)\right].
    \label{eq:lq-one-sided-cost}
\end{equation}
Let $C_T^{(1),\star}(a)$ be the parameter-aware chance-safe oracle cost and
write
\[
    R_T^{(1)}(\pi;a)
    :=
    C_T^{(1)}(\pi;a)-C_T^{(1),\star}(a).
\]
The hidden security margin evolves as
\begin{equation}
    d_{t+1}=d_t-a^\top u_t,
    \qquad
    d_1=d_0>0,
    \qquad
    d_t\ge0.
    \label{eq:lq-one-sided-margin}
\end{equation}
Under $a_0$, the feasible policy $p_t=w$, $q_t=0$ gives
\begin{equation}
    C_T^{(1),\star}(a_0)
    \le
    \lambda w^2T+O(1).
    \label{eq:lq-one-sided-oracle}
\end{equation}

Define commitment relative to $a_1$ by
\begin{equation}
    \tau_1^\pi
    :=
    \inf\left\{
        t\le T:
        \sum_{s=1}^t p_s>d_0
    \right\},
    \qquad
    A_1^\pi:=\{\tau_1^\pi\le T\}.
    \label{eq:lq-one-sided-commitment}
\end{equation}
Commitment violates safety under $a_1$, so
$\mathbb P_{a_1}^\pi(A_1^\pi)\le\eta_T$.  On noncommitment,
$\sum_t p_t\le d_0$.  Since
\[
    x_{t+1}
    =
    \bar x-\sum_{s=1}^t\rho^{t-s}p_s,
\]
the deviations from $\bar x$ sum to at most $d_0/(1-\rho)$, and every
noncommitting trajectory satisfies
\[
    \sum_{t=1}^T x_{t+1}^2
    \ge
    \bar x^2T-\frac{2\bar x d_0}{1-\rho}.
\]
Thus the common-regime gap can be chosen so that
\begin{equation}
    G_T^{(1)}
    \ge
    (\bar x^2-\lambda w^2)T-O(1),
    \label{eq:lq-one-sided-gap}
\end{equation}
and hence $G_T^{(1)}=\Theta(T)$ whenever
$\lambda w^2<\bar x^2$.  The input box also gives
\[
    |x_t|
    \le
    \max\left\{\bar x,\frac{U-w}{1-\rho}\right\},
\]
which verifies the bounded stage-cost claim above.

Under $a_0$ and $a_1$, the residual means are $q_t$ and $p_t$, respectively,
so
\begin{equation}
    \mathrm{kl}_t(a_0,a_1)
    =
    \frac{(q_t-p_t)^2}{2\sigma^2}.
    \label{eq:lq-one-sided-kl}
\end{equation}
Before commitment, $\sum p_t\le d_0$ pathwise.  On the target-safe event,
$\sum q_t\le d_0$; on its complement, which has probability at most
$\eta_T$, only the input box remains.  Splitting the stopped-KL expectation
over these events gives
\begin{equation}
    I_{\mathrm{pre}}^{(1)}(T;a_0,a_1)
    \le
    \frac{\min\{4d_0^2,\,2Ud_0\}}{2\sigma^2}
    +
    \frac{TU^2\eta_T}{2\sigma^2}.
    \label{eq:lq-one-sided-information}
\end{equation}
For $\eta_T\le T^{-2}$, the profile is uniformly bounded, and
Theorem~\ref{thm:binary-precommitment-limit} gives
\[
    R_T^{(1)}(\pi;a_0)=\Omega(T).
\]

\paragraph{Continuous cross-loading.}
The same one-sided model also yields a useful continuous-class calculation.
Replace the binary loading class by
\begin{equation}
    a_\gamma=(\gamma,1)^\top,
    \qquad
    \gamma\in[0,\bar\gamma],
    \label{eq:lq-cross-loading-vector}
\end{equation}
with target $\gamma=0$, so that
\begin{equation}
    d_{t+1}=d_t-q_t-\gamma p_t,
    \qquad
    y_t=q_t+\gamma p_t+\xi_t.
    \label{eq:lq-cross-loading-model}
\end{equation}
For an alternative $\gamma>0$, let commitment be the first time
$\sum_{s\le t}(q_s+\gamma p_s)>d_0$.  This is a safety violation under
$\gamma$, so the commitment allowance is at most $\eta_T$.  Strictly before
commitment,
\[
    \sum_t p_t\le\frac{d_0}{\gamma}.
\]
The $q$-channel cancels from the likelihood ratio, and the same budget also
bounds the target value on noncommitment.  Therefore
\begin{equation}
    I_{\mathrm{pre}}(T;0,\gamma)
    \le
    \frac{\gamma Ud_0}{2\sigma^2},
    \qquad
    G_T^{(1)}(\gamma)
    \ge
    cT-\frac{B}{\gamma}-O(1),
    \label{eq:lq-cross-loading-bounds}
\end{equation}
where
\[
    c:=\bar x^2-\lambda w^2>0,
    \qquad
    B:=\frac{2\bar x d_0}{1-\rho}.
\]
Choose the horizon-dependent alternative $\gamma_T=T^{-1/2}$.  Then
$I_{\mathrm{pre}}(T;0,\gamma_T)=O(T^{-1/2})$ and
$G_T^{(1)}(\gamma_T)\ge cT-O(\sqrt T)$.  For $\eta_T=T^{-2}$,
\[
    \psi_{\eta_T}
    \!\left(I_{\mathrm{pre}}(T;0,\gamma_T)\right)
    =1-o(1),
\]
so Corollary~\ref{cor:continuous-sup} gives
\begin{equation}
    R_T^{(1)}(\pi;0)\ge cT-o(T).
    \label{eq:lq-cross-loading-regret}
\end{equation}
Thus even a cross-loading that vanishes as $T^{-1/2}$ can block
asymptotically the entire linear regulation advantage.

\paragraph{Variant II: signed regulation with quadratic headroom.}
Let $s_t\in\{-1,+1\}$ be an exogenous up/down regulation request, observed
before the action, included in the pre-action history, and independent of $a$.
Take
\begin{equation}
    u_t=(p_t,q_t)\in[-U,U]^2,
    \qquad
    x_{t+1}=\rho x_t+w s_t-p_t-q_t,
    \qquad
    x_1=0,
    \label{eq:lq-signed-plant}
\end{equation}
where $0\le\rho<1$ and $0<w\le U$, with cost
\begin{equation}
    C_T^{(2)}(\pi;a)
    :=
    \mathbb E_a^\pi
    \sum_{t=1}^T
    \left[x_{t+1}^2+\lambda(p_t^2+q_t^2)\right].
    \label{eq:lq-signed-cost}
\end{equation}
Let $C_T^{(2),\star}(a)$ be the parameter-aware chance-safe oracle cost and
write
\[
    R_T^{(2)}(\pi;a)
    :=
    C_T^{(2)}(\pi;a)-C_T^{(2),\star}(a).
\]
The hidden margin now evolves as
\begin{equation}
    d_{t+1}=d_t-(a^\top u_t)^2,
    \qquad
    d_1=d_0>0,
    \qquad
    d_t\ge0.
    \label{eq:lq-signed-margin}
\end{equation}
Under $a_0$, the feasible policy $p_t=ws_t$, $q_t=0$ keeps $x_t=0$ and
certifies $C_T^{(2),\star}(a_0)\le\lambda w^2T$.  Define
\begin{equation}
    \tau_2^\pi
    :=
    \inf\left\{
        t\le T:
        \sum_{s=1}^t p_s^2>d_0
    \right\}.
    \label{eq:lq-signed-commitment}
\end{equation}
Commitment violates safety under $a_1$, so
$\mathbb P_{a_1}^\pi(\tau_2^\pi\le T)\le\eta_T$.

On a target-safe, noncommitting trajectory,
\[
    \|p_{1:T}+q_{1:T}\|_2\le2\sqrt{d_0}.
\]
With
\[
    e_t:=ws_t-p_t-q_t=x_{t+1}-\rho x_t,
\]
we obtain
\[
    \|e_{1:T}\|_2
    \ge
    \bigl(w\sqrt T-2\sqrt{d_0}\bigr)_+,
    \qquad
    \|e_{1:T}\|_2
    \le
    (1+\rho)\|x_{2:T+1}\|_2.
\]
Therefore the common-regime gap satisfies
\begin{equation}
    G_T^{(2)}
    \ge
    \left[
        \frac{\bigl(w\sqrt T-2\sqrt{d_0}\bigr)_+^2}{(1+\rho)^2}
        -\lambda w^2T
    \right]_+.
    \label{eq:lq-signed-gap}
\end{equation}
If $\lambda<(1+\rho)^{-2}$, then $G_T^{(2)}=\Theta(T)$.  Moreover,
\[
    |x_t|\le\frac{w+2U}{1-\rho},
\]
so the stage cost is uniformly bounded here as well.

The one-step divergence remains
\[
    \mathrm{kl}_t(a_0,a_1)
    =
    \frac{(q_t-p_t)^2}{2\sigma^2}.
\]
Before $\tau_2^\pi$, $\sum p_t^2\le d_0$ pathwise.  On the target-safe
event, $\sum q_t^2\le d_0$; on its complement, the input box gives
$\sum(q_t-p_t)^2\le4U^2T$.  Hence
\begin{equation}
    I_{\mathrm{pre}}^{(2)}(T;a_0,a_1)
    \le
    \frac{2d_0}{\sigma^2}
    +
    \frac{2U^2T\eta_T}{\sigma^2}.
    \label{eq:lq-signed-information}
\end{equation}

\begin{corollary}[Linear regret with an unknown active constraint direction]
\label{cor:unknown-direction-linear-regret}
For \(i\in\{1,2\}\), let
\(\Pi_{\mathrm{safe}}^{T,i}\) denote the policies that are uniformly
chance-\(\eta_T\) safe for \(\{a_0,a_1\}\) in Variant \(i\).
Suppose that all system parameters are fixed independently of \(T\) and
that \(\eta_T\leq T^{-2}\).

For Variant I, assume
\(
    g_1
    :=
    \bar{x}^{\,2}-\lambda w^2
    >0,
    \,\,\,
    \beta_1
    :=
    \frac{\min\{4d_0^2,\,2Ud_0\}}{2\sigma^2}.
\)
For Variant II, assume
\(
    g_2
    :=
    w^2\bigl((1+\rho)^{-2}-\lambda\bigr)
    >0,
    \,\,\,
    \beta_2
    :=
    \frac{2d_0}{\sigma^2}.
\)
Then, for \(i\in\{1,2\}\),
\[
    \liminf_{T\to\infty}
    \inf_{\pi\in\Pi_{\mathrm{safe}}^{T,i}}
    \frac{R_T^{(i)}(\pi;a_0)}{T}
    \geq
    \frac{g_i}{2}e^{-\beta_i}
    >0.
\]
In particular, there are constants \(c_i>0\) and \(T_i<\infty\),
depending only on the fixed system parameters, such that every
\(\pi\in\Pi_{\mathrm{safe}}^{T,i}\) satisfies
\[
    R_T^{(i)}(\pi;a_0)\geq c_iT,
    \qquad
    T\geq T_i.
\]
Thus both safety geometries force linear oracle regret under uniform
chance safety.
\end{corollary}

\noindent The proof appears in Appendix. \\
\noindent \textbf{Role of signed actuation.}
Variant~I isolates the one-sided order mismatch and supplies the local
continuous cross-loading calculation.  Variant~II shows that neither the
information ceiling nor the linear-regret conclusion depends on one-sided
actuation: the request and both reserve injections may change sign, while the
active loading spends quadratic energy.  Sign reversal can undo a linear
signed expenditure, but it cannot undo an even expenditure: reversing the
active loading pays the quadratic headroom a second time.  In both variants,
full observation of the regulated state is insufficient because the active
security channel is visible only through the input-dependent residual.  A
persistent direct measurement of the margin would change the stopped
information profile and can remove the obstruction.

%% file: safe_v15_package/positive_v5_edited.tex
\section{Safe Oracle Recovery from Repeatable Experiments}
\label{sec:positive-information-recovery}

Theorem~\ref{thm:binary-precommitment-limit} gives the lower-bound
certificate, and Section~\ref{sec:mechanisms-examples} verifies it in several
systems.  We now give a sufficient condition for recovery.  If a
safety-preserving experiment can be repeated and remains informative, a
confidence-based controller can learn without leaving the safe set.
The result uses three properties:
\begin{enumerate}[leftmargin=*,nosep]
    \item a conservative experiment is safe and returns the system to a common
    continuation state;
    \item repeated experiments accumulate information about the unknown
    parameter; and
    \item the loss in oracle value is controlled by the parameter error.
\end{enumerate}
This is not a converse to the lower bound.  It identifies one setting in which
pre-commitment information is persistent and can be used safely.

\subsection{Repeatable safe experiments}
\label{subsec:repeatable-information-model}

Let
\(
    \theta\in[\underline\theta,\overline\theta],
    \,\,\,
    0<\underline\theta<\overline\theta,
\)
be a fixed unknown scalar parameter.  Time is divided into episodes.  Let
$\mathcal G_{t-1}$ denote the $\sigma$-field generated by the first $t-1$
episodes.  The learner chooses a $\mathcal G_{t-1}$-measurable conservative
design parameter
$\vartheta_t$ and executes an experiment $\kappa_{\vartheta_t}$.  The
experiment returns reward $r_t$ and an observation
\begin{equation}
    y_t=\theta a_t+\xi_t,
    \qquad
    \xi_t\stackrel{\mathrm{iid}}{\sim}\mathcal N(0,\sigma^2),
    \label{eq:positive-observation}
\end{equation}
where $a_t$ is $\mathcal G_{t-1}$-measurable and $\xi_t$ is independent of
$\mathcal G_{t-1}$.  Define the accumulated design information
\begin{equation}
    V_t:=\sum_{s=1}^t\frac{a_s^2}{\sigma^2},
    \qquad V_0:=0.
    \label{eq:positive-information-process}
\end{equation}

Let $v^\star(\theta)\in[0,r_{\max}]$ be the oracle value per episode.  For
$\vartheta\le\theta$, let $v_{\mathrm{bnd}}(\theta,\vartheta)$ be the expected
value of $\kappa_\vartheta$ under the true parameter $\theta$.

\begin{assumption}[Repeatable safe experiments]
\label{ass:positive-repeatability}
There is a nondecreasing function
$\omega:[0,\infty)\to[0,r_{\max}]$ with $\omega(0)=0$ such that:

\begin{enumerate}
    \item[(P1)] \emph{Fixed parameter.}
    The experiments do not change $\theta$, and the same parameter governs all
    episodes.

    \item[(P2)] \emph{Safe repeatability.}
    For every
    $\underline\theta\le\vartheta\le\theta\le\overline\theta$, executing
    $\kappa_\vartheta$ is safe under $\theta$ and, possibly after a fixed
    reset, returns the system to a common state from which the next experiment
    can be run safely.

    \item[(P3)] \emph{Value modulus.}
    Whenever $\vartheta\le\theta$,
    \begin{equation}
        0
        \le
        v^\star(\theta)-v_{\mathrm{bnd}}(\theta,\vartheta)
        \le
        \omega(\theta-\vartheta).
        \label{eq:positive-value-modulus}
    \end{equation}

    \item[(P4)] \emph{Bounded reward.}
    Almost surely, $0\le r_t\le r_{\max}$, and
    \begin{equation}
        \mathbb E_\theta[r_t\mid\mathcal G_{t-1}]
        =v_{\mathrm{bnd}}(\theta,\vartheta_t)
        \quad\text{whenever }\vartheta_t\le\theta.
        \label{eq:positive-conditional-value}
    \end{equation}

    \item[(P5)] \emph{Informative initialization.}
    When $\kappa_{\underline\theta}$ is used in the first episode,
    $a_1\ne0$ almost surely.  Hence $V_t>0$ almost surely for every $t\ge1$.
\end{enumerate}
\end{assumption}

Thus conservative operation is itself a repeatable experiment.  As the design
parameter approaches the truth, its expected value approaches the oracle value.

\subsection{Lower-confidence policy}
\label{subsec:positive-confidence}

For $t\ge1$, define the weighted least-squares estimate
\begin{equation}
    \widehat\theta_t
    :=
    \frac{\sum_{s=1}^t a_sY_s/\sigma^2}{V_t}.
    \label{eq:positive-wls}
\end{equation}
For a mixture scale $\nu>0$ and confidence level $\delta\in(0,1)$, let
\begin{equation}
    b(v,\delta;\nu)
    :=
    \left
    \{
      \frac{
        2(v+\nu)
        \log\!\left(
          \frac{\sqrt{1+v/\nu}}{\delta}
        \right)
      }{v^2}
    \right\}^{1/2},
    \qquad v>0.
    \label{eq:positive-mixture-radius}
\end{equation}

\begin{lemma}[Information-indexed confidence sequence]
\label{lem:positive-mixture-confidence}
Under \eqref{eq:positive-observation}, for every $\delta\in(0,1)$ and
$\nu>0$,
\begin{equation}
    \mathbb P_\theta\!\left(
      |\widehat\theta_t-\theta|
      \le b(V_t,\delta;\nu)
      \text{ for every }t\ge1
    \right)
    \ge1-\delta.
    \label{eq:positive-confidence-event}
\end{equation}
\end{lemma}

The proof is in Appendix~\ref{app:positive-confidence-proof}.  The radius is
indexed by the accumulated information $V_t$ and is of order
$\sqrt{\log V_t/V_t}$ for large $V_t$.

The controller uses the lower endpoint of this confidence sequence:
\begin{equation}
    \vartheta_1:=\underline\theta,
    \qquad
    \vartheta_t
    :=
    \left[
      \widehat\theta_{t-1}
      -b(V_{t-1},\delta;\nu)
    \right]_{\underline\theta}^{\overline\theta},
    \quad t\ge2,
    \label{eq:positive-lcb-policy}
\end{equation}
where
\[
    [x]_{\underline\theta}^{\overline\theta}
    :=\min\{\overline\theta,\max\{\underline\theta,x\}\}.
\]
At episode $t$, it runs $\kappa_{\vartheta_t}$, observes $(r_t,Y_t)$, and
updates \eqref{eq:positive-wls}.

\subsection{Safe recovery theorem}
\label{subsec:positive-general-theorem}

The parameter-aware oracle receives $v^\star(\theta)$ per episode.  In this
episodic specialization, its expected regret is
\begin{equation}
    R_T(\pi;\theta)
    :=
    \sum_{t=1}^T
    \left(
      v^\star(\theta)-\mathbb E_\theta^\pi[r_t]
    \right).
    \label{eq:positive-episodic-regret}
\end{equation}

\begin{theorem}[Safe recovery from repeatable information]
\label{thm:positive-information-recovery}
Suppose Assumption~\ref{ass:positive-repeatability} holds.  Then the policy
\eqref{eq:positive-lcb-policy} satisfies, for every
$\theta\in[\underline\theta,\overline\theta]$,
\begin{equation}
    \overline P_{\theta,T}^\pi(\mathsf{Safe}_{\theta,T}^c)
    \le\delta,
    \label{eq:positive-theorem-safety}
\end{equation}
and
\begin{equation}
    R_T(\pi;\theta)
    \le
    r_{\max}
    +
    \sum_{t=2}^T
    \mathbb E_\theta^\pi\!\left[
      \omega\!\left(
        2b(V_{t-1},\delta;\nu)
      \right)
    \right]
    +
    r_{\max}(T-1)\delta.
    \label{eq:positive-information-regret-bound}
\end{equation}
Consequently, for $\eta_T\in(0,1)$, setting $\delta=\eta_T$ gives
$\pi\in\Pi_{\mathrm{safe}}^T([\underline\theta,\overline\theta])$.
\end{theorem}

The proof is in Appendix~\ref{app:positive-theorem-proof}.  Whenever
$|\widehat\theta_{t-1}-\theta|\le b(V_{t-1},\delta;\nu)$, the projection in
\eqref{eq:positive-lcb-policy} gives $\vartheta_t\le\theta$ and
$\theta-\vartheta_t\le2b(V_{t-1},\delta;\nu)$.  Thus (P2) gives safety and
(P3) bounds the loss in that episode.  Lemma~\ref{lem:positive-mixture-confidence}
makes these inequalities simultaneous with probability at least $1-\delta$;
bounded reward controls the remaining event.

\subsection{Rates from information growth}
\label{subsec:positive-rates}

Theorem~\ref{thm:positive-information-recovery} converts a lower bound on
$V_t$ into a regret rate.

\begin{corollary}[Polynomial information growth]
\label{cor:positive-information-growth}
Suppose Assumption~\ref{ass:positive-repeatability} holds and, uniformly over
$\theta\in[\underline\theta,\overline\theta]$, there are constants
$c_-,c_+>0$ and $\alpha\in(0,1]$ such that
\begin{equation}
    c_-t^\alpha
    \le V_t\le c_+t
    \qquad\text{almost surely for every }t\ge1.
    \label{eq:positive-information-growth-assumption}
\end{equation}
Assume also that $\omega(r)\le Lr^q$ for some $L<\infty$ and $q>0$.  Run the
policy with $\delta=\delta_T=T^{-2}$ and fixed $\nu>0$.  Then
\begin{equation}
    \sup_{\theta\in[\underline\theta,\overline\theta]}
    R_T(\pi;\theta)
    =
    \begin{cases}
    O\!\left(
      L T^{1-\alpha q/2}(\log T)^{q/2}
    \right),
    & \alpha q<2,
    \\[1mm]
    O\!\left(
      L(\log T)^{1+q/2}
    \right),
    & \alpha q=2,
    \\[1mm]
    O\!\left(
      L(\log T)^{q/2}
    \right),
    & \alpha q>2.
    \end{cases}
    \label{eq:positive-rate-taxonomy}
\end{equation}
The hidden constants may depend on $c_-$, $c_+$, $\alpha$, $\nu$, $q$, and
$r_{\max}$, but not on $T$ or $\theta$.
\end{corollary}

The proof is in Appendix~\ref{app:positive-rate-proof}.  For example, if
$V_t\asymp t$ and $\omega(r)\le \ell r$, then
\begin{equation}
    R_T(\pi;\theta)=O\!\left(\ell\sqrt{T\log T}\right).
    \label{eq:positive-lipschitz-persistent-rate}
\end{equation}
If $V_t\asymp t$ and $\omega(r)\le Lr^2$, then
\begin{equation}
    R_T(\pi;\theta)=O\!\left(L(\log T)^2\right).
    \label{eq:positive-quadratic-value-rate}
\end{equation}
The information-growth rate and the value modulus jointly determine regret.

\subsection{Alignment as a sufficient condition}
\label{subsec:positive-alignment}

\begin{definition}[Safety--information alignment]
\label{def:positive-alignment}
The family
$\{\kappa_\vartheta\}_{\vartheta\in[\underline\theta,\overline\theta]}$ is
\emph{aligned} if the same experiment that satisfies the safety and
repeatability condition (P2) also produces the observation \eqref{eq:positive-observation}.  It is uniformly
$(\rho,\bar\rho)$-aligned if every execution satisfies
\begin{equation}
    0<\rho
    \le
    \frac{a_t^2}{\sigma^2}
    \le
    \bar\rho<\infty
    \qquad\text{almost surely.}
    \label{eq:positive-uniform-alignment}
\end{equation}
\end{definition}

Under alignment, no separate excitation phase is needed: the maneuver used to
preserve safety also supplies information.

\begin{corollary}[Uniform alignment]
\label{cor:positive-uniform-alignment}
Suppose Assumption~\ref{ass:positive-repeatability} holds, the experiment
family is uniformly $(\rho,\bar\rho)$-aligned, and
$\omega(r)\le\ell r$.  Then the lower-confidence policy with
$\delta=T^{-2}$ is uniformly chance-$T^{-2}$ safe and
\begin{equation}
    \sup_{\theta\in[\underline\theta,\overline\theta]}
    R_T(\pi;\theta)
    =
    O\!\left(
      \ell\sqrt{\frac{T\log T}{\rho}}
    \right).
    \label{eq:positive-uniform-alignment-rate}
\end{equation}
The hidden constant may depend on $\bar\rho$, $\nu$, and $r_{\max}$.
\end{corollary}

\subsection{Example: recursively feasible braking}
\label{subsec:positive-braking}

Consider a vehicle approaching a wall at $p=0$.  Its fixed but unknown braking
efficiency satisfies
\[
    \theta\in[\underline\theta,\overline\theta].
\]
Each episode starts at $p(0)=-D$.  The controller selects the initial speed,
which is also the episode reward.  Since $\theta\le\overline\theta$, the
reward is bounded by $r_{\max}:=\sqrt{2D\overline\theta}$.  The controller
then applies braking according to
\begin{equation}
    \dot p(s)=v(s),
    \qquad
    \dot v(s)=-\theta u(s),
    \qquad
    u(s)\in[0,1],
    \label{eq:positive-braking-dynamics}
\end{equation}
Safety requires $p(s)\le0$ until the vehicle stops.  A fixed reset then returns
the vehicle to the same initial state for the next episode.

For known $\theta$, maximal braking from speed $v$ has stopping distance
$v^2/(2\theta)$.  The oracle entry speed is therefore
\begin{equation}
    v^\star(\theta)=\sqrt{2D\theta}.
    \label{eq:positive-braking-oracle}
\end{equation}
A conservative experiment uses
\begin{equation}
    v(\vartheta)=\sqrt{2D\vartheta},
    \label{eq:positive-braking-boundary-speed}
\end{equation}
applies maximal braking, and records the noisy deceleration measurement
\begin{equation}
    y_t=\theta+\xi_t,
    \qquad
    \xi_t\stackrel{\mathrm{iid}}{\sim}\mathcal N(0,\sigma^2).
    \label{eq:positive-braking-observation}
\end{equation}
Thus $a_t=1$ in \eqref{eq:positive-observation}.

\begin{lemma}[Braking verifies alignment]
\label{lem:positive-braking-alignment}
For every
$\underline\theta\le\vartheta\le\theta\le\overline\theta$, the experiment
\eqref{eq:positive-braking-boundary-speed} is safe under $\theta$, returns to
the common initial state after the reset, contributes Fisher information
$1/\sigma^2$, and satisfies
\begin{equation}
    0
    \le
    v^\star(\theta)-v(\vartheta)
    \le
    \sqrt{\frac{D}{2\underline\theta}}
    (\theta-\vartheta).
    \label{eq:positive-braking-modulus}
\end{equation}
\end{lemma}

\begin{corollary}[Safe oracle recovery for braking]
\label{cor:positive-braking-recovery}
The lower-confidence policy applied to the braking instance is uniformly
chance-$T^{-2}$ safe and satisfies
\begin{equation}
    \sup_{\theta\in[\underline\theta,\overline\theta]}
    R_T(\pi;\theta)
    =
    O\!\left(
      \sigma
      \sqrt{\frac{D}{\underline\theta}}
      \sqrt{T\log T}
    \right).
    \label{eq:positive-braking-regret}
\end{equation}
\end{corollary}

Each safe braking episode supplies one new observation, so
$V_t=t/\sigma^2$.  Unlike the finite-expenditure mechanism in
Lemma~\ref{lem:finite-expenditure}, the experiment can be repeated without
consuming a nonrenewable safety margin.

\subsubsection{Nominal certainty equivalence at an active safety boundary}
\label{rem:braking_certainty_equivalence}
The braking example also separates statistical consistency from trajectory
safety. Suppose that, after \(n\) safe braking episodes, the controller forms
the projected sample mean
\(
    \widehat{\theta}^{\mathrm{CE}}_n
    :=
    \left[
        \frac{1}{n}\sum_{s=1}^n Y_s
    \right]_{\underline{\theta}}^{\overline{\theta}}
\)
and then uses the certainty-equivalent entry speed
\(
    v_{n+1}^{\mathrm{CE}}
    :=
    \sqrt{2D\widehat{\theta}^{\mathrm{CE}}_n}.
\)
For any
\(\theta\in(\underline{\theta},\overline{\theta})\), the resulting
episode is safe if and only if
\(\widehat{\theta}^{\mathrm{CE}}_n\leq\theta\), because
\[
    \frac{(v_{n+1}^{\mathrm{CE}})^2}{2\theta}
    =
    D\frac{\widehat{\theta}^{\mathrm{CE}}_n}{\theta}.
\]
Since
\[
    \frac{1}{n}\sum_{s=1}^n Y_s-\theta
    \sim
    \mathcal{N}\!\left(0,\frac{\sigma^2}{n}\right),
\]
and projection preserves comparison with an interior value of \(\theta\),
\[
    \mathbb{P}_{\theta}\!\left(
        \frac{(v_{n+1}^{\mathrm{CE}})^2}{2\theta}>D
    \right)
    =
    \mathbb{P}_{\theta}\!\left(
        \widehat{\theta}^{\mathrm{CE}}_n>\theta
    \right)
    =
    \frac{1}{2}
\]
for every \(n\).

More generally, consider the confidence-adjusted plug-in rule
\[
    \widehat{\theta}_{n,m}
    :=
    \left[
        \frac{1}{n}\sum_{s=1}^n Y_s-m_n
    \right]_{\underline{\theta}}^{\overline{\theta}},
    \qquad
    v_{n+1}(m_n)
    :=
    \sqrt{2D\widehat{\theta}_{n,m}},
\]
where \(m_n\geq 0\). For an interior parameter,
\[
    \mathbb{P}_{\theta}\!\left(
        \frac{v_{n+1}(m_n)^2}{2\theta}>D
    \right)
    =
    1-\Phi\!\left(\frac{\sqrt{n}\,m_n}{\sigma}\right),
\]
where \(\Phi\) is the standard Gaussian distribution function. Thus a
one-episode violation probability at most \(\delta<1/2\) requires
\[
    m_n
    \geq
    \frac{\sigma}{\sqrt{n}}\,
    \Phi^{-1}(1-\delta).
\]
Increasing \(n\) therefore shrinks the magnitude of the
certainty-equivalent overshoot, but does not shrink its probability when
\(m_n=0\). A trajectory-wide guarantee requires a time-uniform,
data-dependent margin, which is the role of the confidence sequence in the
lower-confidence policy.

\subsection{Relation to pre-commitment information}
\label{subsec:positive-precommitment-connection}

Definition~\ref{def:pairwise-pre-information} and
Lemma~\ref{lem:general-stopped-kl} identify pre-commitment information with a
stopped sum of one-step divergences.  In the present Gaussian model, this sum
is a multiple of $V_T$.

\begin{proposition}[Pre-commitment KL for repeatable experiments]
\label{prop:positive-information-equals-kl}
Fix $\theta\ne\theta'$ and a policy that is uniformly safe for
$\{\theta,\theta'\}$ and executes experiments of the form
\eqref{eq:positive-observation} for $T$ episodes without entering a
model-specific commitment regime.  Then
\begin{equation}
    D_{\mathrm{KL}}\!\left(
      P_{\theta,T}^{\pi,\mathrm{pre}}
      \,\middle\|\,
      P_{\theta',T}^{\pi,\mathrm{pre}}
    \right)
    =    \frac{(\theta-\theta')^2}{2}
    \mathbb E_\theta^\pi[V_T].
    \label{eq:positive-kl-information-identity}
\end{equation}
Consequently, if such a policy satisfies $V_T\ge g(T)$ almost surely, then
\begin{equation}
    I_{\mathrm{pre}}(T;\theta,\theta')
    \ge
    \frac{(\theta-\theta')^2}{2}g(T).
    \label{eq:positive-profile-lower-bound}
\end{equation}
Under uniform $\rho$-alignment, repeating the common-safe experiment
$\kappa_{\underline\theta}$ remains noncommitting and therefore
\begin{equation}
    I_{\mathrm{pre}}(T;\theta,\theta')
    \ge
    \frac{\rho}{2}(\theta-\theta')^2T.
    \label{eq:positive-linear-profile}
\end{equation}
\end{proposition}

The lower bound requires bounded information for every uniformly safe policy
that has not committed.  Here one common-safe experiment gives linear
information growth, so the bounded-information obstruction is absent.  The
recovery theorem also shows how to use that information through a safe
confidence rule and convert estimation error into value loss.

\begin{remark}[Scope]
\label{rem:positive-scope}
The result treats a scalar parameter with a known order in which smaller design
values are conservative.  It also requires repeatable experiments and a common
reset or continuation state.  These are sufficient conditions, not necessary
ones.  Vector parameters, multiple control-equivalence classes, non-resetting
trajectories, and dynamic latent environments require additional transition and
exploitation arguments.
\end{remark}

%% file: safe_v15_package/computable_v3_edited_final.tex
\section{Computing Pre-Commitment Information}
\label{sec:certification}

Theorem~\ref{thm:binary-precommitment-limit} requires an upper bound on the
information available before commitment.  In general, computing this quantity
is an adaptive experiment-design problem.  We now give an exact reduction for
a deterministic linear--Gaussian class with quadratic common-safe
constraints.

This section addresses only the information calculation.  The commitment
event and the value lost without commitment must still be derived from the
control problem.  Exact formulas can also show that common-safe information
grows, but growth alone does not provide the safe transition required by
Section~\ref{sec:positive-information-recovery}.  For the list problem in
Section~\ref{subsec:list-limit}, pairwise certificates give a conservative
bound on the corresponding multiclass information.

\subsection{A deterministic linear--Gaussian class}
\label{subsec:computable-open-loop}

Fix a pair $\theta,\theta'\in\Theta_0$ and write
\[
    g:=\theta-\theta'.
\]
Let $\mathcal U_{\mathrm{core}}^T$ be the set of complete action sequences
that remain in the common pre-commitment regime through horizon $T$.

\begin{assumption}[Deterministic linear--Gaussian certificate class]
\label{ass:deterministic-design}
The following hold for the pair $(\theta,\theta')$.
\begin{enumerate}
    \item[(D1)] Under $\vartheta\in\{\theta,\theta'\}$, the observations obey
    the sequential model
    \begin{equation}
        y_t=\vartheta^\top\zeta_t(u_{1:t})+\xi_t,
        \qquad
        \xi_t\stackrel{\mathrm{iid}}{\sim}\mathcal N(0,\sigma^2),
        \label{eq:computable-observation-model}
    \end{equation}
    where $\zeta_t$ is a deterministic causal function, common to the two
    environments, and $\xi_t$ is independent of the past and of the policy's
    private randomization.

    \item[(D2)] Membership in $\mathcal U_{\mathrm{core}}^T$ depends only on
    the action sequence, not on the observation noise.

    \item[(D3)] Every action prefix that can occur strictly before commitment
    under a uniformly safe policy extends to a sequence in
    $\mathcal U_{\mathrm{core}}^T$.  Conversely, every sequence in
    $\mathcal U_{\mathrm{core}}^T$ can be played open loop by a uniformly safe
    policy without commitment.
\end{enumerate}
\end{assumption}

This class includes the consumable-resource examples in which the regressor is
a known function of the controls.  It does not cover general partially
observed stochastic systems, where noise or feedback changes which future
experiments remain feasible.

For $u_{1:T}\in\mathcal U_{\mathrm{core}}^T$, define
\begin{equation}
    \Phi_T(u_{1:T})
    :=
    \sum_{t=1}^T
    \bigl(g^\top\zeta_t(u_{1:t})\bigr)^2.
    \label{eq:computable-phi}
\end{equation}

\begin{proposition}[Open-loop information reduction]
\label{lem:open-loop-reduction}
Under Assumption~\ref{ass:deterministic-design},
\begin{equation}
    I_{\mathrm{pre}}(T;\theta,\theta')
    =
    \frac{1}{2\sigma^2}
    \sup_{u_{1:T}\in\mathcal U_{\mathrm{core}}^T}
    \Phi_T(u_{1:T}).
    \label{eq:computable-open-loop}
\end{equation}
Thus feedback and randomization do not increase the maximum pre-commitment
information in this class.  The proof is in
Appendix~\ref{app:computable-open-loop-proof}.
\end{proposition}

Equation~\eqref{eq:computable-open-loop} is a constrained experiment-design
problem: choose a common-safe action sequence to maximize signal energy. The unconstrained counterpart is
classical deterministic input design, where open-loop sequences with
certified excitation can be constructed explicitly
\cite{Willems05,Saligrama12chirp}; the constraint that the sequence remain in
$\mathcal U^T_{\mathrm{core}}$ is what caps the value of the program.

\subsection{Quadratic information programs}
\label{subsec:computable-quadratic-programs}

Suppose the regressors are linear in the stacked action vector
\[
    \mathbf u
    :=
    \begin{bmatrix}
      u_1^\top&\cdots&u_T^\top
    \end{bmatrix}^\top
    \in\mathbb R^{n_T},
    \qquad
    \zeta_t(\mathbf u)=Z_{t,T}\mathbf u,
\]
where causality means that $Z_{t,T}$ has zero columns for controls after time
$t$.  Define
\begin{equation}
    M_T
    :=
    \sum_{t=1}^T
    Z_{t,T}^\top gg^\top Z_{t,T}
    \succeq0.
    \label{eq:computable-MT}
\end{equation}
Then $\Phi_T(\mathbf u)=\mathbf u^\top M_T\mathbf u$.

\paragraph{Single quadratic budget.}
Suppose
\begin{equation}
    \mathcal U_{\mathrm{core}}^T
    =
    \left\{
       \mathbf u:
       \mathbf u^\top K_T\mathbf u\le\epsilon_T
    \right\},
    \qquad
    K_T\succeq0,
    \quad
    \epsilon_T>0.
    \label{eq:computable-single-core}
\end{equation}
Define the generalized Rayleigh quotient (with value zero when $M_T=0$)
\begin{equation}
    \lambda_{\max}(M_T,K_T)
    :=
    \sup_{\mathbf u:\,\mathbf u^\top K_T\mathbf u>0}
    \frac{\mathbf u^\top M_T\mathbf u}
         {\mathbf u^\top K_T\mathbf u}.
    \label{eq:computable-generalized-eigenvalue}
\end{equation}

\begin{proposition}[Single-budget certificate]
\label{prop:computable-single-budget}
Assume $M_T,K_T\succeq0$.
\begin{enumerate}
    \item If $\ker K_T\not\subseteq\ker M_T$, then the information program is
    unbounded: the model contains an informative direction with zero quadratic
    safety cost.

    \item If $\ker K_T\subseteq\ker M_T$, then
    \begin{equation}
        I_{\mathrm{pre}}(T;\theta,\theta')
        =
        \frac{\epsilon_T}{2\sigma^2}
        \lambda_{\max}(M_T,K_T).
        \label{eq:computable-spectral}
    \end{equation}
\end{enumerate}
The proof is in Appendix~\ref{app:computable-single-budget-proof}.
\end{proposition}

The profile is therefore the available quadratic safety resource multiplied by
the largest ratio of information to resource.  Bounded information follows when
\[
    \epsilon_T\lambda_{\max}(M_T,K_T)
\]
is uniformly bounded; linear growth follows when this product is $\Theta(T)$.

\begin{corollary}[Stationary matched-order systems]
\label{cor:computable-block-stationary}
Suppose
\[
    M_T=I_T\otimes M_0,
    \qquad
    K_T=I_T\otimes K_0,
\]
and $\ker K_0\subseteq\ker M_0$.  Then
\begin{equation}
    I_{\mathrm{pre}}(T;\theta,\theta')
    =
    \frac{\epsilon_T}{2\sigma^2}
    \lambda_{\max}(M_0,K_0).
    \label{eq:computable-block-stationary}
\end{equation}
A bounded $\epsilon_T$ gives a bounded profile.  If
$\epsilon_T=\Theta(T)$ and $\lambda_{\max}(M_0,K_0)>0$, the profile grows
linearly.  The proof is in
Appendix~\ref{app:computable-block-stationary-proof}.
\end{corollary}

\subsection{Multiple constraints and an SDP certificate}
\label{subsec:computable-sdp}

Suppose the common core is contained in the quadratic outer approximation
\begin{equation}
    \mathcal U_{\mathrm{core}}^T
    \subseteq
    \widehat{\mathcal U}_{\mathrm{core}}^T
    :=
    \left\{
       \mathbf u:
       \mathbf u^\top Q_{j,T}\mathbf u\le b_{j,T},
       \quad j=1,\ldots,m_T
    \right\},
    \label{eq:computable-multi-core}
\end{equation}
where $Q_{j,T}\succeq0$ and $b_{j,T}\ge0$.  Then
\begin{equation}
    2\sigma^2 I_{\mathrm{pre}}(T;\theta,\theta')
    \le
    v_T
    :=
    \max_{\mathbf u}
    \left\{
       \mathbf u^\top M_T\mathbf u:
       \mathbf u^\top Q_{j,T}\mathbf u\le b_{j,T},\ j\in[m_T]
    \right\}.
    \label{eq:computable-qcqp}
\end{equation}
The QCQP is nonconvex in general.  Its semidefinite dual is
\begin{equation}
\begin{aligned}
    d_T^{\mathrm{sdp}}
    :=
    \min_{\lambda\ge0}
    \quad&
    \sum_{j=1}^{m_T}b_{j,T}\lambda_j
    \\
    \text{subject to}\quad&
    \sum_{j=1}^{m_T}\lambda_jQ_{j,T}\succeq M_T.
\end{aligned}
    \label{eq:computable-dual-sdp}
\end{equation}
Weak duality gives $v_T\le d_T^{\mathrm{sdp}}$.  Thus every feasible dual
point is a valid information upper bound; strong duality is not required.

\begin{theorem}[Uniform SDP information certificate]
\label{thm:computable-uniform-sdp}
Suppose that for every horizon $T$ there are multipliers
$\lambda_{j,T}\ge0$ such that
\begin{equation}
    \sum_{j=1}^{m_T}\lambda_{j,T}Q_{j,T}\succeq M_T
    \label{eq:computable-matrix-domination}
\end{equation}
and
\begin{equation}
    \sum_{j=1}^{m_T}b_{j,T}\lambda_{j,T}\le\Lambda
    \qquad\text{for all }T
    \label{eq:computable-uniform-dual-value}
\end{equation}
for some finite $\Lambda$.  Then
\begin{equation}
    \beta_{\mathrm{pre}}(\theta,\theta')
    \le
    \frac{\Lambda}{2\sigma^2}.
    \label{eq:computable-uniform-capacity}
\end{equation}
Once the commitment allowance and the noncommitment regret profile have been bounded, Theorem~\ref{thm:binary-precommitment-limit} applies with information ceiling $\Lambda/(2\sigma^2)$.  The proof is in Appendix~\ref{app:computable-uniform-sdp-proof}.
\end{theorem}

A finite-horizon solve bounds only $I_{\mathrm{pre}}(T)$.  A bounded-capacity
certificate requires feasible multipliers whose objective values are uniformly
bounded in $T$.  Conversely, a growing SDP value does not prove that the true
information profile grows: the relaxation may be loose.  Even exact
information growth does not by itself establish safe oracle recovery.

\begin{remark}[Affine terms and input constraints]
Affine dynamics, nonzero initial states, off-center ellipsoids, and linear
input constraints can be represented by homogenizing with
$\widetilde{\mathbf u}=(1,\mathbf u^\top)^\top$ and imposing
$\widetilde u_0=1$ (or $W_{00}=1$ after lifting).  The resulting SDP remains
a valid upper certificate, but the relaxation may be more conservative.
\end{remark}

\subsection{Labeled problems}
\label{subsec:computable-labeled}

Pairwise certificates also give a conservative upper bound on the mutual
information used in Proposition~\ref{prop:list-fano}.  Let
$\Theta_M=\{\theta_1,\ldots,\theta_M\}$ and use the same labeled stopped record
$S_T^\pi$ under every label.  Suppose that every
$\pi\in\Pi_{\mathrm{safe}}^T(\Theta_M)$ satisfies
\[
    D_{\mathrm{KL}}\!\left(
      P_{\theta_j,T}^{\pi,\mathrm{pre}}
      \,\middle\|\,
      P_{\theta_k,T}^{\pi,\mathrm{pre}}
    \right)
    \le\overline\beta_{jk,T},
    \qquad j,k\in[M].
\]
Then
\begin{equation}
    I_{\mathrm{pre}}^{(M)}(T;\Theta_M)
    \le
    \frac{1}{M^2}
    \sum_{j=1}^M\sum_{k=1}^M
    \overline\beta_{jk,T}.
    \label{eq:computable-multiclass-average}
\end{equation}
The proof is in Appendix~\ref{app:computable-multiclass-bridge-proof}.  The
bound can be loose because the pairwise programs may favor different sensing
actions.  For the unknown-location problem in
Section~\ref{subsec:unknown-location-lists}, the direct list calculation is
sharper.

\subsection{Finite-headroom example}
\label{subsec:computable-thermal}

Recall the finite-headroom model used in the binary consequence of
Theorem~\ref{thm:binary-precommitment-limit}:
\[
    d_{t+1}=d_t-\theta u_t^2,
    \qquad
    y_t=\theta u_t+\xi_t,
    \qquad
    0\le u_t\le U.
\]
For the pair $(\theta_b,\theta_d)$ with
$0\le\theta_b<\theta_d$, the common regime satisfies
\[
    \sum_{t=1}^T u_t^2\le\frac{d_0}{\theta_d}.
\]
The information matrix is
\[
    M_T=(\theta_d-\theta_b)^2I_T,
    \qquad
    K_T=I_T.
\]
With the pointwise action cap, the largest feasible input energy is
$\min\{TU^2,d_0/\theta_d\}$.  Hence
\begin{equation}
    I_{\mathrm{pre}}(T;\theta_b,\theta_d)
    =
    \frac{(\theta_d-\theta_b)^2}{2\sigma^2}
    \min\left\{TU^2,\frac{d_0}{\theta_d}\right\}.
    \label{eq:computable-thermal-exact}
\end{equation}
Once $TU^2\ge d_0/\theta_d$, the profile saturates at
$(\theta_d-\theta_b)^2d_0/(2\sigma^2\theta_d)$, which is the bound used in
the finite-headroom consequence.

The one-sided budget in Remark~\ref{rem:signed-parity} is better
handled directly.  If $0\le u_t\le U$ and $\sum_tu_t\le B$, then
\[
    \sum_tu_t^2\le\min\{B^2,UB\}.
\]
No quadratic-budget reformulation is needed.

\paragraph{Scope of the computational certificate.}
A feasible finite-horizon dual point bounds one profile value.  A uniformly
bounded family of dual points certifies bounded pre-commitment information.
Pairwise bounds can also be combined into a sufficient multiclass certificate.
A growing SDP upper bound does not establish persistent information, and none
of these calculations supplies the commitment value or the safe transition to
the oracle; those remain separate control arguments.

%% file: safe_v15_package/numerical_v3_edited.tex
\section{Numerical Illustrations}
\label{sec:numerical-illustrations}

The preceding sections give three outcomes.  Bounded pre-commitment
information yields an oracle-gap lower bound through
Theorem~\ref{thm:binary-precommitment-limit}.  Repeatable information aligned
with safety yields oracle recovery through
Theorem~\ref{thm:positive-information-recovery}.  In the multi-alternative
setting of Section~\ref{subsec:unknown-location-lists}, finite sensing can
instead support a certified list and partial adaptation.

This section illustrates these outcomes.  The replenishment and sensor
examples evaluate analytic information profiles.  The unknown-location example
uses Monte Carlo only to evaluate the resulting list sizes and values.  The braking example compares the simulated lower-confidence policy
with the exact robust baseline, nominal certainty equivalence, and
the analytic regret upper bound from Theorem~\ref{thm:positive-information-recovery}.
Monte Carlo is used only to display policy performance; the safety
guarantee for the lower-confidence policy remains analytic. Safety in all four examples is established
analytically, not by simulation.  The quantities \(I_{\rm av}\),
\(I_{\rm rate}\), and \(I_{\rm level}\) are example-specific evidence
quantities used in the plots; \(I_{\rm pre}\) in
Definition~\ref{def:pairwise-pre-information} remains the general profile.

\subsection{Replenishment: from bounded to persistent safe information}
\label{subsec:numerics-replenishment}

Consider the replenishing thermal resource
\begin{equation}
    d_{t+1}
    =
    \min\{d_{\max},d_t+b-\theta u_t^2\},
    \qquad
    d_t\ge0,
    \qquad
    0\le u_t\le U,
    \label{eq:num-replenishing-dynamics}
\end{equation}
with sensor
\begin{equation}
    y_t=\theta u_t+\xi_t,
    \qquad
    \xi_t\stackrel{\mathrm{iid}}{\sim}\mathcal N(0,\sigma^2).
    \label{eq:num-replenishing-sensor}
\end{equation}
We compare the target \(\theta_0=0\) with the incompatible alternative
\(\theta_1>0\).  The \(\theta_0\)-oracle uses \(u_t=U\).  Before excluding
\(\theta_1\), the learner must preserve feasibility under \(\theta_1\).

The largest Gaussian divergence available over a length-\(T\) common-safe
experiment is
\begin{equation}
    I_{\rm av}(T;b)
    =
    \frac{\theta_1^2}{2\sigma^2}
    \min\left\{
      U^2T,
      \frac{d_0+bT}{\theta_1}
    \right\}.
    \label{eq:num-replenishing-profile}
\end{equation}
The first term is the action limit; the second is the initial headroom plus
replenishment.  Figure~\ref{fig:num-replenishment} uses
\[
    \theta_0=0,
    \quad
    \theta_1=1,
    \quad
    \sigma=1,
    \quad
    U=1,
    \quad
    d_0=d_{\max}=2.
\]
In the plotted horizon range, the resource term is active and
\[
    I_{\rm av}(T;b)=1+\frac b2T.
\]
Thus \(b=0\) gives a bounded information budget, whereas every fixed \(b>0\)
gives a positive information rate.

A robust-then-commit policy accumulates evidence inside the common-safe regime
and uses \(u=U\) only after the log-likelihood ratio in favor of \(\theta_0\)
over \(\theta_1\) reaches
\begin{equation}
    \kappa_T=\log(1/\eta_T)=2\log T,
    \qquad
    \eta_T=T^{-2}.
    \label{eq:num-llr-stopping}
\end{equation}
Under \(\theta_1\), the exponentiated likelihood ratio is a nonnegative
martingale.  Ville's inequality therefore bounds the probability of false
commitment by \(e^{-\kappa_T}=T^{-2}\).

\begin{figure}[!htbp]
\centering
\includegraphics[width=\textwidth]{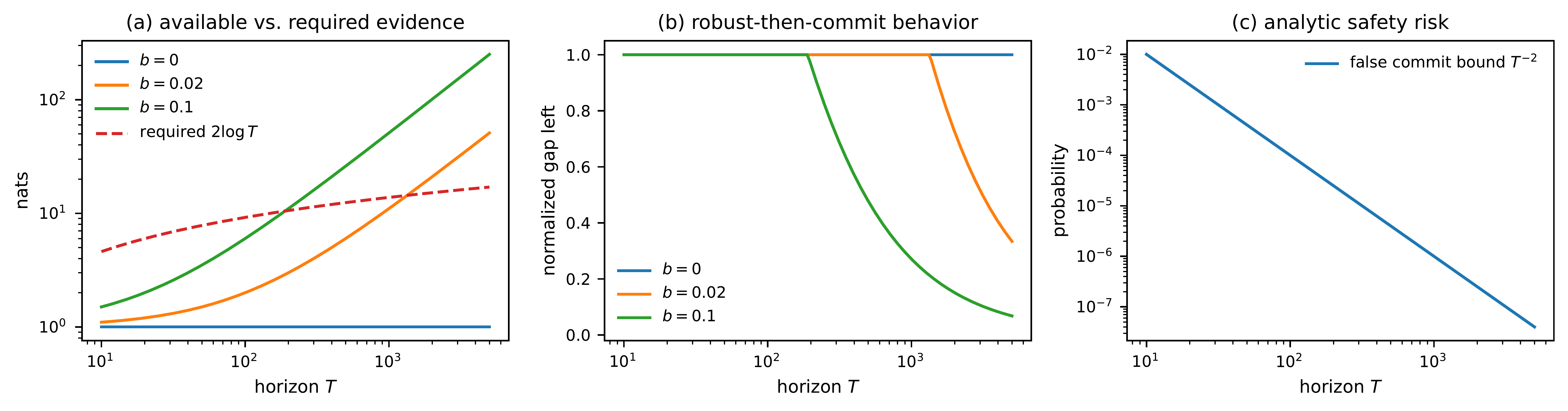}
\caption{\textbf{Replenishment changes the information-growth regime} (Sec~\ref{subsec:numerics-replenishment}).
Replenishing finite-headroom device with target $\theta_0=0$,
alternative $\theta_1=1$, and replenishment rate $b$.
\emph{(a)} Available common-safe evidence $I_{\mathrm{av}}(T;b)$
against the commitment threshold $\kappa_T=2\log T$: bounded for
$b=0$, linear for every fixed $b>0$.
\emph{(b)} Normalized gap left, $\min\{1,\kappa_T/I_{\mathrm{av}}(T;b)\}$:
the fraction of the horizon a robust-then-commit policy spends below the
threshold, and hence the fraction of the commitment-only gap it forfeits.
It is an evidence-budget ratio, not a simulated regret; it equals one when
the threshold is unreachable and vanishes when the threshold requires a
negligible fraction of the available evidence.
\emph{(c)} The Ville bound $e^{-\kappa_T}=T^{-2}$ on false commitment under
$\theta_1$; the curve is analytic, not a Monte Carlo estimate.}
\label{fig:num-replenishment}
\end{figure}

For \(b=0\), increasing the horizon does not increase the evidence available
before commitment.  For \(b>0\), the information profile grows linearly, so
the bounded-information premise of Theorem~\ref{thm:binary-precommitment-limit}
no longer applies.

\subsection{Sensor ablation: transient evidence versus persistent imprint}
\label{subsec:numerics-sensor}

We next keep the plant, objective, and safe diagnostic action sequence fixed
and change only the sensor.  Set \(b=0\) and retain the pair
\((\theta_0,\theta_1)=(0,\theta_1)\).  The learner applies a safe prefix with
\[
    q:=\sum_{t=1}^m u_t^2\le d_0/\theta_1,
\]
and then sets \(u_t=0\).

With the rate sensor
\[
    y_t^{\rm rate}=\theta u_t+\xi_t,
\]
no further information arrives after the prefix.  For every \(T\ge m\),
\begin{equation}
    I_{\rm rate}(T)=\frac{\theta_1^2q}{2\sigma^2}.
    \label{eq:num-rate-information}
\end{equation}
With the level sensor
\[
    y_t^{\rm level}=d_t+\xi_t,
\]
the same prefix leaves a persistent separation
\(\Delta_d=\theta_1q\) between the two headroom levels.  Independent passive
measurements then give
\begin{equation}
    I_{\rm level}(T)
    =
    I_{\rm level}(m)
    +
    \frac{\Delta_d^2}{2\sigma^2}(T-m),
    \qquad T\ge m.
    \label{eq:num-level-information}
\end{equation}
Thus the rate sensor yields finite evidence, while the level sensor converts
the same finite perturbation into a linearly growing observation record.

\begin{figure}[!htbp]
\centering
\includegraphics[width=.86\textwidth]{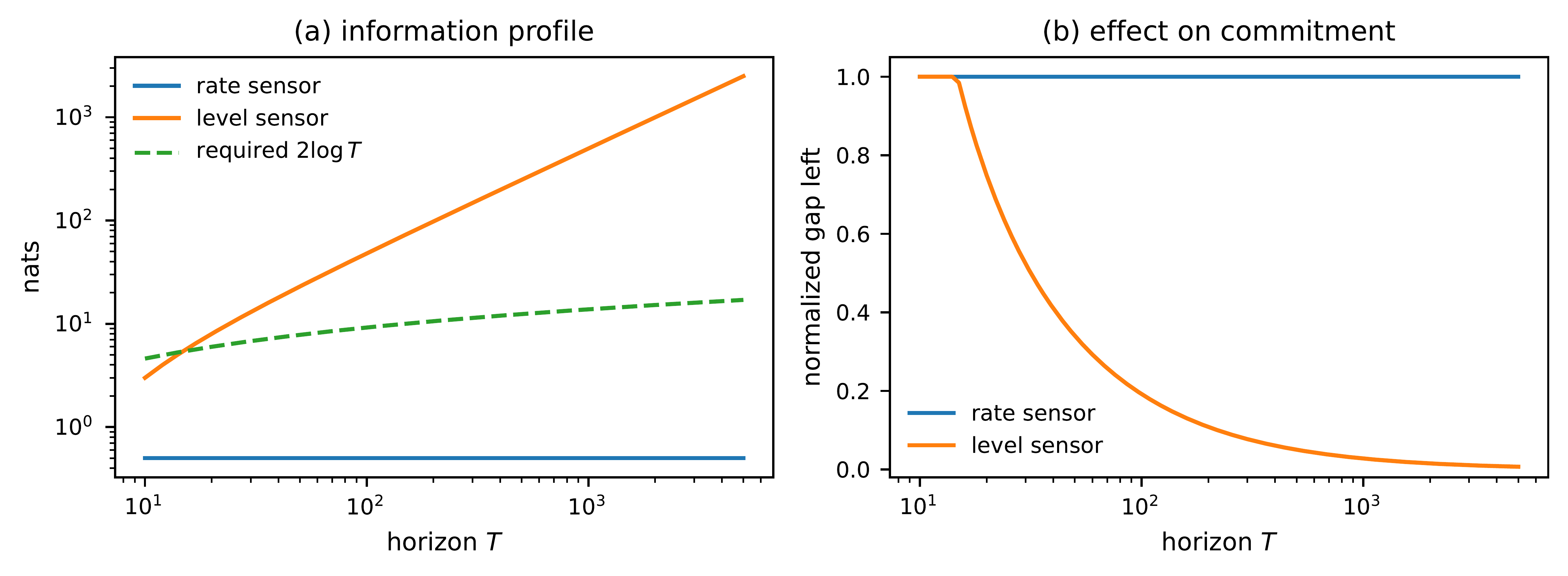}
\caption{\textbf{The sensing channel changes the information profile} (Section~\ref{subsec:numerics-sensor}); both configurations apply the same safe
diagnostic prefix of energy $q$, and only the sensor differs.
\emph{(a)} Pre-commitment information: the rate sensor
$y^{\mathrm{rate}}_t=\theta u_t+\xi_t$ records only the transient input,
giving the bounded profile; the level sensor
$y^{\mathrm{level}}_t=d_t+\xi_t$ observes the persistent headroom
separation $\Delta_d=\theta_1 q$, so information grows linearly.
\emph{(b)} Normalized gap left, $\min\{1,\kappa_T/I(T)\}$, as in
Figure~2(b): the rate sensor never reaches the growing threshold
$\kappa_T=2\log T$, while the level sensor reaches it after
logarithmically many passive measurements.}
\label{fig:num-sensor}
\end{figure}

The example isolates the role of observation design: the plant dynamics alone
do not determine the pre-commitment information regime.

\subsection{Unknown disturbance location: safe lists and partial adaptation}
\label{subsec:numerics-unknown-location}

We use the multi-alternative model from
Section~\ref{subsec:unknown-location-lists}.  There are \(M\) possible affected
locations and one true location \(J\).  During the diagnostic window,
\begin{equation}
    x_{i,t+1}=x_{i,t}+\Delta\mathbf 1\{i=J\},
    \qquad
    x_{i,1}=0,
    \qquad
    x_{i,t}\le H,
    \label{eq:num-location-diagnostic-state}
\end{equation}
with \(H=N\Delta\).  Hence unit-throughput diagnostic operation is safe for
\(N\) rounds.  Varying \(N\) changes the available physical headroom, not
only a policy parameter; the curves therefore compare different safe sensing
windows.  At each round, the learner inspects one location:
\begin{equation}
    y_t
    =
    \mu\mathbf 1\{A_t=J\}+\xi_t,
    \qquad
    \xi_t\stackrel{\rm iid}{\sim}\mathcal N(0,\sigma^2).
    \label{eq:num-location-sensor}
\end{equation}

At the deadline, the learner returns a nonempty certified list
\(\mathcal C_N\).  The post-diagnostic controller divides mitigation and
throughput across the retained locations.  If \(J\in\mathcal C_N\), it arrests
the affected state.  If \(J\notin\mathcal C_N\), the state violates the
constraint on the next round and the continuation reward is set to zero.
Therefore safety reduces to retaining the true location.

We use the elimination rule from Section~\ref{subsec:unknown-location-lists}.
A candidate is deleted only when its likelihood-ratio statistic exceeds
\(\kappa_T=\log(1/\eta_T)\).  Since the lists only shrink, Ville's inequality
gives, for every true location \(J\),
\begin{equation}
    \mathbb P_J(J\notin\mathcal C_N)\le\eta_T.
    \label{eq:num-location-list-safety}
\end{equation}
At each round, the policy inspects the least-sampled surviving location.  This
balanced rule is used to demonstrate partial certification; it is not claimed
to be an optimal sensing policy.

All three controllers receive unit reward during the \(N\) diagnostic rounds.
Afterward, the robust controller retains all \(M\) locations, the list
controller uses \(\mathcal C_N\), and the oracle retains only \(J\).  Their
values are
\begin{align}
    J_T^{\rm robust}(N)
    &=
    N+\frac{T-N}{M},
    \label{eq:num-location-robust-value}
    \\
    J_T^{\rm list}(N)
    &=
    N
    +(T-N)
    \mathbb E\!\left[
       \frac{\mathbf 1\{J\in\mathcal C_N\}}
            {|\mathcal C_N|}
    \right],
    \label{eq:num-location-list-value}
    \\
    J_T^{\rm oracle}(N)
    &=T.
    \label{eq:num-location-oracle-value}
\end{align}
The fraction of the robust-to-oracle gap recovered by the list policy is
\begin{equation}
    \Gamma_N
    :=
    \frac{J_T^{\rm list}(N)-J_T^{\rm robust}(N)}
         {J_T^{\rm oracle}(N)-J_T^{\rm robust}(N)}.
    \label{eq:num-location-recovered-fraction}
\end{equation}

Panel~(a) of Figure~\ref{fig:num-unknown-location} reports
\[
    I(J;S_N)=\log M-\mathbb E[H(p_N)],
\]
where \(p_N\) is the posterior computed from the exact Gaussian likelihoods
under a uniform prior.  We use
\[
    M=8,
    \quad
    \mu=2.5,
    \quad
    \sigma=1,
    \quad
    T=1000,
    \quad
    \eta_T=T^{-2}=10^{-6}.
\]
For each value of \(N\), all Monte Carlo quantities are averaged over
\(10^5\) independent draws of \(J\) and the sensor noise.  The standard error
of every reported reward rate or probability is therefore at most
\(1/(2\sqrt{10^5})<1.6\times10^{-3}\).

At \(N=36\), the reward rates are
\[
    \text{robust}=0.1565,
    \qquad
    \text{safe list}=0.3893,
    \qquad
    \text{oracle}=1.
\]
The list policy recovers \(27.6\%\) of the robust-to-oracle gap.  Its mean list
size is \(3.62\); \(10.6\%\) of trials produce a singleton list and
\(88.9\%\) produce a proper partial list.  No false elimination was observed
in \(10^5\) trials.  The safety claim is the analytic bound
\eqref{eq:num-location-list-safety}, not the empirical count.

\begin{figure}[t]
\centering
\includegraphics[width=\textwidth]{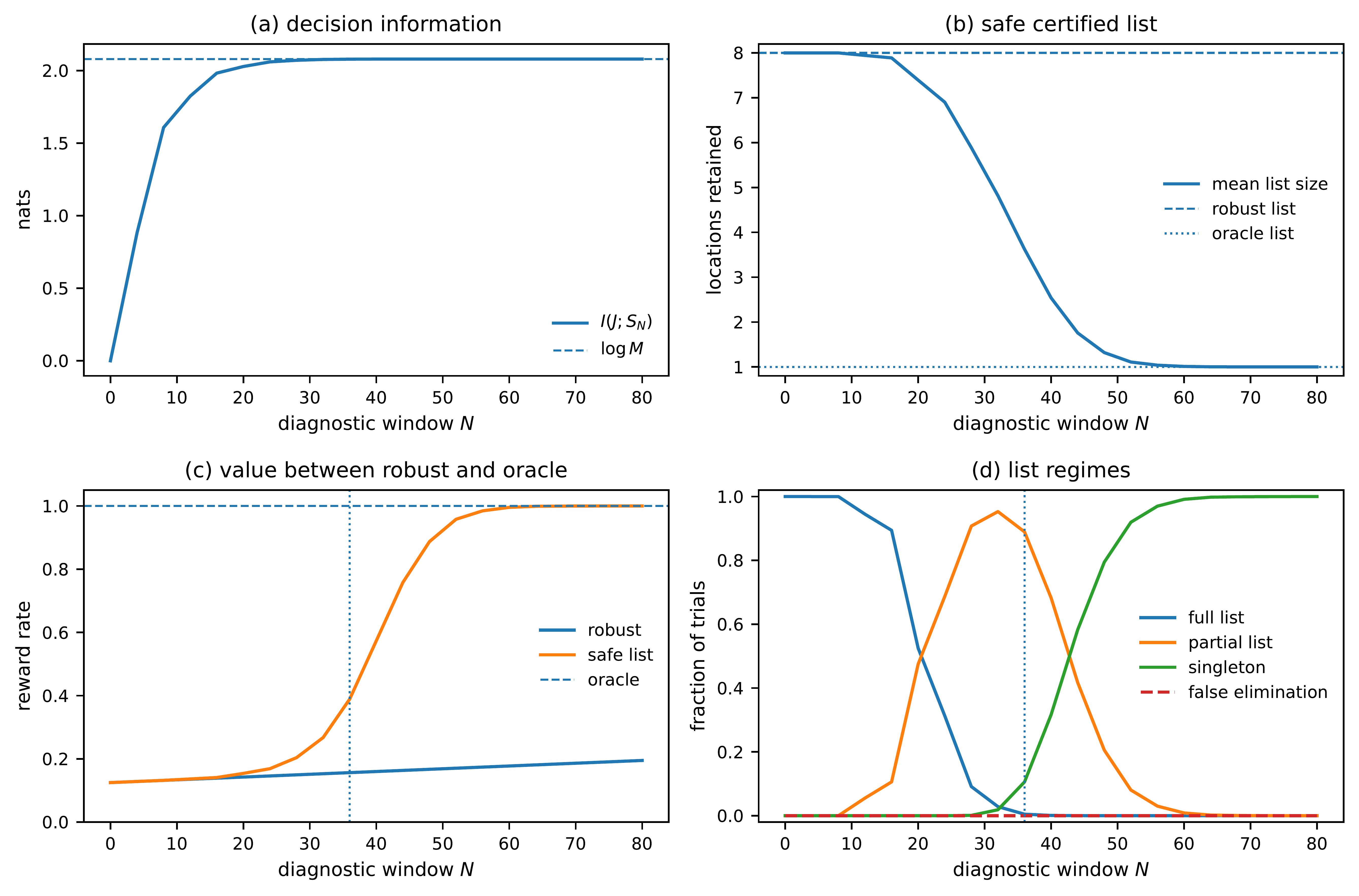}
\caption{\textbf{Limited sensing produces safe partial adaptation.}
Unknown-disturbance-location model (Section~\ref{subsec:numerics-unknown-location}) with $M=8$ locations, true
location $J$, and a diagnostic window of $N$ safe rounds; parameters as in
the text.
\emph{(a)} Decision information $I(J;S_N)$ approaches $\log M$ as $N$ grows.
\emph{(b)} Mean size of the certified list $\mathcal{C}_N$, shrinking from
the robust list $[M]$ toward the oracle list $\{J\}$.
\emph{(c)} Post-diagnostic reward rates: the list policy lies
between the robust and oracle values over a broad range of $N$.
\emph{(d)} Fraction of trials returning the full list, a proper partial
list, or a singleton; near the operating point $N=36$ most trials produce
a proper partial list rather than exact localization. No false elimination
occurred in $10^5$ trials; the safety claim is the analytic bound,
not the empirical count.}
\label{fig:num-unknown-location}
\end{figure}

This example realizes the intermediate case quantified by
Proposition~\ref{prop:list-fano}: a smaller certified uncertainty set can
improve value before exact localization is possible.

\subsection{Repeatable braking: safe information converted into performance}
\label{subsec:numerics-braking}

The final illustration uses the braking model of
Section~\ref{subsec:positive-braking}.  The safety-preserving maneuver is also
informative.  We set
\(
    \theta\in[0.5,1.5],
    \,\,\,
    \theta^\star=1,
    \,\,\,
    D=1,
    \,\,\,
    \sigma=0.3.
\)
Each episode begins at distance \(D\).  Given a pessimistic parameter
\(\vartheta_t\), the learner enters with speed
\[
    v(\vartheta_t)=\sqrt{2D\vartheta_t},
\]
applies maximal braking, and observes
\(
    y_t=\theta^\star+\xi_t.
\)
Because the design coefficient is one in every episode, the accumulated
information is
\(    V_t=t/\sigma^2. \)

The learner uses the lower-confidence rule in
Section~\ref{subsec:positive-confidence} with \(\delta_T=T^{-2}\).  On the
simultaneous confidence event, \(\vartheta_t\le\theta^\star\) for every
\(t\), so the stopping-distance calculation gives safety episode by episode.
Corollary~\ref{cor:positive-braking-recovery} gives
\begin{equation}
    R_T
    =
    O\!\left(
      \sigma
      \sqrt{\frac{D}{\underline\theta}}
      \sqrt{T\log T}
    \right).
    \label{eq:num-braking-rate}
\end{equation}
A permanently robust controller uses \(v(\underline\theta)\), so its regret at
\(\theta^\star>\underline\theta\) is linear.

We fix a terminal horizon $T_{\max}=5000$, set
$\delta=T_{\max}^{-2}$ and $\nu=1/\sigma^2$, and simulate
$20{,}000$ independent trajectories in Figure~\ref{fig:num-braking}. All quantities at episode
$t\leq T_{\max}$ are prefixes of this single horizon-$T_{\max}$
policy. For nominal certainty equivalence, no violation through episode
$t=n+1$ is equivalent to the first $n$ centered Gaussian partial
sums remaining nonpositive. For continuous symmetric increments,
the corresponding survival probability is
\[
    \mathbb{P}_{\theta^\star}
    \bigl(\text{no CE violation through episode }t\bigr)
    =
    \frac{\binom{2n}{n}}{4^n},
    \,\,\, n=t-1.
\implies
    \mathbb{P}_{\theta^\star}
    \bigl(\text{CE violation by episode }t\bigr)
    =
    1-\frac{\binom{2n}{n}}{4^n}.
\]
\begin{figure}[!h]
\centering
\includegraphics[width=.9\textwidth]{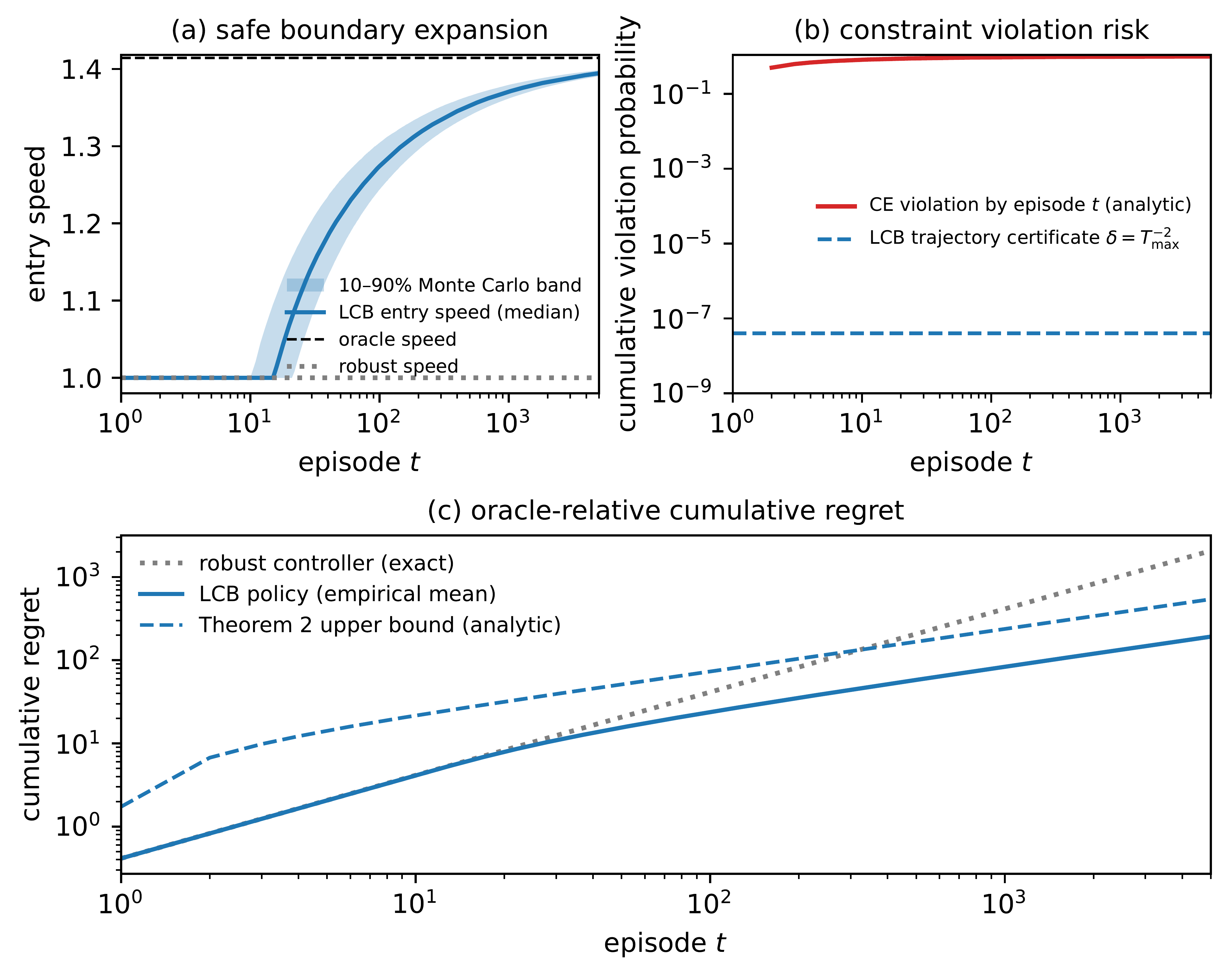}
\caption{\textbf{Repeatable information must be converted into a
safety-valid decision rule.}
Recursively feasible braking with
$\theta\in[0.5,1.5]$, $\theta^\star=1$, $D=1$,
$\sigma=0.3$, $T_{\max}=5000$,
$\delta=T_{\max}^{-2}$, and $\nu=1/\sigma^2$.
\textbf{(a)} Median entry speed of the lower-confidence policy over
$20{,}000$ independent trajectories, with the $10$--$90$ percentile
band. The policy initially coincides with the robust controller and
then expands the certified safe boundary toward the oracle.
\textbf{(b)} Cumulative probability that nominal certainty equivalence
has violated the stopping-distance constraint by episode $t$, compared
with the trajectory-wide LCB (lower confidence bound policy (see Sec.~\ref{subsec:positive-confidence}) certificate $\delta$.
\textbf{(c)} Exact cumulative regret of the permanently robust
controller, empirical mean regret of the LCB policy, and the analytic
upper bound from Theorem~\ref{thm:positive-information-recovery}. Nominal certainty
equivalence is omitted from the regret comparison because it is not
uniformly safe. No LCB violations were observed in simulation; the
safety claim is the analytic bound
$\mathbb{P}(\mathrm{Safe}^{c})\leq\delta$.}
\label{fig:num-braking}
\end{figure}

This example verifies the two ingredients used in
Section~\ref{sec:positive-information-recovery}: one safe policy produces
persistent information, and the value loss decreases with the confidence
radius.

%% file: safe_v15_package/related_v4_constraints_bite.tex
\section{Related Work}
\label{sec:related}

\noindent \textbf{Adaptive control and constrained LQR.}
Finite-time adaptive-control results obtain regret or sample-complexity bounds
when deliberate excitation, process noise, or closed-loop state variation
keeps the system informative
\cite{AbbasiYadkori11,DeanFoCM20,CohenKorenMansour19,SimchowitzFoster20}. Certainty equivalence can be efficient in unconstrained LQR when a point
estimate may be used directly for control~\cite{ManiaTuRecht19}.
A binding trajectory constraint changes this implication: an optimistic
estimation error can itself make the selected action unsafe before later
observations can correct the estimate. Section~\ref{subsec:numerics-braking} gives a
simple instance.
Learning-based MPC can separate a robust model used for safety from a learned
model used for performance, with convergence to the model-informed controller
under sufficient excitation~\cite{AswaniEtAl13}.  Constrained-LQR results give
more specific guarantees.  Dean et al.~\cite{DeanConstrained19} combine persistent excitation with robust
controller synthesis.  Li et al.~\cite{LiDasShammaLi21} obtain
$\widetilde O(T^{2/3})$ regret for a single-trajectory constrained problem.
Schiffer and Janson~\cite{SchifferJanson25} show that a binding safety
constraint can itself generate useful state variation in a one-dimensional
problem with nonlinear baselines.  Hutchinson et al.~\cite{HutchinsonJiangAlizadeh26} obtain
$\widetilde O(\sqrt T)$ regret for multidimensional chance-constrained LQR.
These results give conditions under which safe operation continues to supply
information.  In our terms, they place the problem in a persistent
precommitment-information regime.  We ask whether this property holds for a
given dynamics--sensor--constraint geometry.  Bounded precommitment
information gives a certificate that it does not.

\noindent \textbf{Dual control, experiment design, and controlled sensing.}
Dual control accounts for the effect of an action on both current performance
and future information~\cite{Feldbaum60,AstromWittenmark}.  Experiment design
for control selects inputs for their value to the eventual control task
\cite{hjalmarsson2005experiment}.  Sequential experimental design and
controlled sensing select experiments to distinguish among hypotheses
\cite{Chernoff59,NaghshvarJavidi13,NitinawaratAtiaVeeravalli13,GopalanLS21}.  Our setting
adds a decision-time constraint.  The action that uses the information must be
chosen before the observation generated by that action is available.  Safety
may also shrink the set of admissible experiments as margin is spent.  Thus
full-horizon information can count evidence that arrives too late, and the
relevant likelihood ratio is stopped before commitment.

\noindent \textbf{Safe learning, recursive feasibility, and returnability.}
Safe learning and learning-based control are surveyed by Garc\'ia and
Fern\'andez~\cite{GarciaFernandez15}, Hewing et al.~\cite{HewingEtAl20}, and
Brunke et al.~\cite{BrunkeEtAl22}.  Recursive feasibility and controlled
invariance are central tools in constrained control
\cite{MayneEtAl00,blanchini1999invariance}.  Safety filters and reachability
methods preserve a safe continuation while a learning controller operates
\cite{fisac2019general,wabersich2021predictive,ames2019control}.  Safe
exploration methods use recoverability, returnability, or terminal safe sets
to avoid states with no safe exit
\cite{MoldovanAbbeel12,turchetta2016safe,KollerEtAl19,berkenkamp2017safe}.
In the deterministic systems considered here, noncommitment relative to an
alternative is exactly preservation of hard-safe recursive feasibility under
that alternative.  Commitment is the first action that removes this
continuation.  Our lower bound asks whether the evidence needed for that
action can be collected while the continuation is still preserved.  Related
methods include safe Bayesian optimization and safe linear bandits
\cite{sui2015safe,amani2019linear}; safety issues in power-system learning are
surveyed by Dobbe et al.~\cite{dobbe2020learning}.

\noindent \textbf{Data informativity.}
Data informativity asks whether a given dataset supports a controller that
works for every model consistent with the data.  This line builds on
persistency of excitation~\cite{Willems05} and has been developed into a
general informativity theory~\cite{vanWaarde20}.
Worst-case identification bounds what a finite data record can resolve when
the model class itself is uncertain
\cite{VenkateshDahleh97,VenkateshDahleh01}, and deterministic input designs
with certified excitation show how such records can be generated
\cite{Saligrama12chirp}. Mixed stochastic--set-membership formulations sharpen the excitation story: persistency of excitation is necessary, not only sufficient, for uniform consistency, and finite-data lower bounds hold without assuming a
consistent estimator~\cite{saligrama2005convex}. Our question is online and
decision-timed: whether safe interaction can generate the data needed for a
model-specific decision before that decision must be made.  The distinction is
between certifying a controller from available data and certifying whether the
required data can be collected safely in time.

%% file: safe_v15_package/discussion_v3_constraints_bite.tex
\section{Discussion and Limitations}
\label{sec:discussion}

\noindent \textbf{When safety changes the learning problem.}
If a near-oracle policy can be reached while preserving a safe continuation
under every plausible model, the safety-specific obstruction is absent.
Information can continue to grow inside the common feasible regime, and the
problem is in the setting addressed by constrained adaptive-control analyses.
The lower-bound examples have a different geometry.  Better target
performance requires leaving the common regime, while the actions that
separate the target from an alternative consume the margin that must be
preserved under that alternative.  The issue is then not only the rate at
which uncertainty shrinks.  It is whether the required evidence is available
before the first action that needs it.

\noindent \textbf{Information timing versus information growth.}
Persistent excitation asks whether information continues to grow along the
closed loop.  Precommitment information asks whether the relevant information
arrives in time.  A system may be identifiable from its complete trajectory
and still fail to support safe adaptation if the observations that distinguish
the relevant models appear only after commitment.  Stopping the record is
therefore part of the causal statement, not only a proof technique.  The
positive theorem treats a case in which the safe maneuver itself is
informative and repeatable, so information accumulates without giving up
recursive feasibility.

\noindent \textbf{Design implications.} When persistent safe information is available, it must still be converted into a safety-valid decision rule. In the braking example, nominal certainty equivalence can overestimate braking authority and violate the constraint, whereas the lower-confidence rule expands the safe operating boundary only as the data justify it. When precommitment information is bounded, however, changing the estimator cannot remove the obstruction; the controlled experiment or the required degree of specialization must change. Replenishment, a persistent sensor signal, and certified candidate lists illustrate these three possibilities.

\noindent \textbf{Limits of the present results.}
Bounded precommitment information is a sufficient certificate for an
unavoidable oracle gap.  Failure of the certificate does not prove recovery.
The pairwise profile may be large because different policies are informative
for different alternatives, even though no single safe policy resolves all of
them.  Information may also be available but unusable because the controller
cannot enter the selected regime without losing feasibility.  The recovery
theorem proves one positive result for an ordered scalar parameter, a
repeatable experiment family, and a common continuation state.  Extending it
to vector parameters, several control-equivalence classes, and nonresetting
dynamics requires additional transition and exploitation arguments.  The
computational certificate omits process noise, and the general case of
unbounded but sublinear precommitment information remains unresolved.

The lower-bound theorem itself is more general.  An application must specify
a predictable commitment rule, compute or bound the allowed commitment
probability under the alternative, evaluate the noncommitment regret profile
under the target, and bound the stopped information.  Pathwise value ceilings
are one way to control the regret profile in deterministic systems; they are
not required by the theorem.

%% file: safe_v15_package/appendix_mechanisms_edited.tex
\section*{Acknowledgments}
The author thanks Aditya Gangrade and Alex Olshevsky for technical
discussions, and Emiliano Dall'Anese for help with practical applications.
Large language models were used for editorial assistance, including
identifying possible issues in exposition, consistency, and references.
The author independently developed and verified all mathematical arguments
and conclusions, checked the cited sources, and assumes full responsibility
for the manuscript.

\section{Proofs for Section~\ref{sec:mechanisms-examples}}
\label{app:mechanism-proofs}

\subsection{Proof of Lemma~\ref{lem:finite-expenditure}}
\label{app:proof-finite-expenditure}

\begin{proof}
Fix a uniformly safe policy and one sample path.  Define
\[
    s_t:=\mathbf 1\{t<\tau^\pi\}S_t.
\]
Then $s_t\ge0$ and, by \eqref{eq:finite-expenditure-budget},
\[
    \sum_{t=1}^T s_t\le B_T.
\]
For $q\ge1$,
\[
    \sum_{t=1}^T s_t^q
    \le
    \left(\sum_{t=1}^T s_t\right)^q
    \le B_T^q.
\]
Using \eqref{eq:finite-expenditure-local-kl},
\[
    \sum_{t=1}^T
    \mathbf 1\{t<\tau^\pi\}\mathrm{kl}^{\theta,\theta'}_t
    \le
    \kappa\sum_{t=1}^T s_t^q
    \le
    \kappa B_T^q.
\]
Take expectation under $\theta$, apply
Lemma~\ref{lem:general-stopped-kl}, and then take the supremum over
uniformly safe policies.  This gives
\eqref{eq:finite-expenditure-information}.
\end{proof}
\subsection{Proof of Corollary~\ref{cor:unknown-direction-linear-regret}}

For Variant I, the noncommitment calculation above gives
\[
    G_T^{(1)}
    \geq
    g_1T-O(1),
\]
while the stopped-information bound gives
\[
    I_{\mathrm{pre}}^{(1)}(T;a_0,a_1)
    \leq
    \beta_1+o(1).
\]
For Variant II,
\[
    G_T^{(2)}
    \geq
    g_2T-O(\sqrt{T}),
    \qquad
    I_{\mathrm{pre}}^{(2)}(T;a_0,a_1)
    \leq
    \beta_2+o(1).
\]
In both variants, commitment under \(a_1\) implies a safety violation,
so the commitment allowance is at most
\(\eta_T=o(1)\). The uniform stage-cost bounds give a
noncommitment-profile remainder of order
\(O(T\eta_T)=o(T)\).
Applying Theorem~\ref{thm:binary-precommitment-limit}, dividing by
\(T\), and taking the lower limit proves
\[
    \liminf_{T\to\infty}
    \inf_{\pi\in\Pi_{\mathrm{safe}}^{T,i}}
    \frac{R_T^{(i)}(\pi;a_0)}{T}
    \geq
    \frac{g_i}{2}e^{-\beta_i}.
\]
The eventual linear lower bound follows because the right-hand side is
strictly positive.
\subsection{Proof of Lemma~\ref{lem:finite-sensing-window}}
\label{app:proof-sensing-window}

\begin{proof}
Fix a uniformly safe policy $\pi$.  The policy action at time $t$, the decision
to continue, and the eventual specialization label are functions of the
previously observed record and environment-independent private randomization.
They therefore add no information about $J$ conditional on that record.  The
chain rule for the stopped record gives
\[
    I_\pi(J;S_T^\pi)
    =
    \sum_{t=1}^T
    P_\pi(t<\tau^\pi)
    I_\pi\!\left(
      J;Z_t\mid H_{t-1},U_t,t<\tau^\pi
    \right).
\]
Only label-dependent observations contribute.  By assumption there are at
most $N_T$ such observations before specialization or fallback, and each
contributes at most $c_{\mathrm{obs}}$.  Hence
\[
    I_\pi(J;S_T^\pi)\le N_Tc_{\mathrm{obs}}.
\]
Taking the supremum over uniformly safe policies proves
\eqref{eq:sensing-window-total}.

For the Gaussian sensor, condition on $(H_{t-1},A_t)$ and define
\[
    B_t:=\mathbf 1\{A_t=J\}.
\]
The observation depends on $J$ only through $B_t$, so
\[
    I(J;Z_t\mid H_{t-1},A_t)
    =I(B_t;Z_t\mid H_{t-1},A_t)
    \le H(B_t\mid H_{t-1},A_t)
    \le\log2.
\]
Let
\[
    p_t:=P(J=A_t\mid H_{t-1},A_t),
    \qquad
    P_1:=\mathcal N(\mu,\sigma^2),
    \qquad
    P_0:=\mathcal N(0,\sigma^2).
\]
The output mixture minimizes the weighted average relative entropy over the
choice of reference distribution.  Taking $P_0$ as a reference gives
\[
    I(B_t;Z_t\mid H_{t-1},A_t)
    \le p_tD_{\mathrm{KL}}(P_1\|P_0)
    \le \frac{\mu^2}{2\sigma^2}.
\]
Combining the two bounds proves
\eqref{eq:sensing-window-cobs}.
\end{proof}

\subsection{Proof of Proposition~\ref{prop:list-fano}}
\label{app:proof-list-fano}

\begin{proof}
Let
\[
    E:=\{J\notin C(S)\ \text{or}\ |C(S)|>L\},
    \qquad
    P(E)=\varepsilon_L.
\]
On $E^c$, the true label belongs to a list of size at most $L$, so
\[
    H(J\mid S,E^c)\le\log L.
\]
On $E$, use the bound $H(J\mid S,E)\le\log M$.  Therefore
\begin{align*}
    H(J\mid S)
    &\le H(\mathbf 1_E\mid S)+H(J\mid S,\mathbf 1_E)\\
    &\le h_2(\varepsilon_L)
      +(1-\varepsilon_L)\log L
      +\varepsilon_L\log M.
\end{align*}
Since $J$ is uniform on $[M]$,
\[
    I(J;S)
    =H(J)-H(J\mid S)
    \ge
    (1-\varepsilon_L)\log\frac{M}{L}
    -h_2(\varepsilon_L),
\]
which proves \eqref{eq:list-fano-bound}.

If
$P(J\notin C(S))\le\eta$ and
$P(|C(S)|>L)\le\delta$, then
$\varepsilon_L\le\eta+\delta$ by the union bound.  For $L<M$, the function
\[
    x\longmapsto (1-x)\log(M/L)-h_2(x)
\]
is decreasing on $[0,1/2]$.  Substituting
$\varepsilon_L\le\eta+\delta\le1/2$ gives
\eqref{eq:list-fano-two-sided}.
\end{proof}

\subsection{Safe-list guarantee and value calculation for
Section~\ref{subsec:unknown-location-lists}}
\label{app:proof-unknown-location-list-values}

\begin{proof}[Proof of the safe-list guarantee]
Fix the true location $J$.  At a time $s$ with $A_s=J$, the multiplicative
increment in $\exp(\Lambda_{J,t})$ is
\[
    \frac{\phi_{0,\sigma}(Z_s)}{\phi_{\mu,\sigma}(Z_s)}.
\]
Under the true model, $Z_s\sim\mathcal N(\mu,\sigma^2)$ conditionally on the
past, so this increment has conditional expectation one.  At times with
$A_s\ne J$, the process does not change.  Hence
$\{\exp(\Lambda_{J,t})\}_{t\ge0}$ is a nonnegative martingale under the true
location, even though the sampled locations are chosen adaptively.  Ville's
inequality gives
\[
    P_J\!\left(
      \sup_{t\le N}\Lambda_{J,t}
      \ge\log\frac1{\eta_T}
    \right)
    \le\eta_T.
\]
The event on the left is exactly false elimination of the true location before
the deadline.  If an empty candidate set is replaced by $[M]$, the failure
probability can only decrease.  This proves
\eqref{eq:location-list-safety-guarantee}.
\end{proof}

\begin{proof}[Proof of the value formulas]
Suppose first that $J\in\mathcal C_N$.  For the affected component,
\[
    x_{J,t+1}-x_{J,t}
    =
    \frac{\Delta}{|\mathcal C_N|}
    -\frac{\Delta}{|\mathcal C_N|}
    =0.
\]
Each retained unaffected component decreases by
$\Delta/|\mathcal C_N|$, and each nonretained unaffected component is
unchanged.  Thus the post-diagnostic controller is safe whenever the list
contains the true location.  If $J\notin\mathcal C_N$, the affected component
starts the post-diagnostic phase at the boundary and increases on the next
round, so a violation occurs immediately.

Every policy earns unit reward in each of the $N$ diagnostic rounds.  The
robust controller retains all $M$ locations and earns post-diagnostic reward
$1/M$ per round.  The list controller earns
$1/|\mathcal C_N|$ per post-diagnostic round when $J\in\mathcal C_N$ and zero
after failure.  The oracle retains the singleton $\{J\}$ and earns unit reward
throughout.  These calculations give
\eqref{eq:location-robust-value}--\eqref{eq:location-oracle-value}.
\end{proof}

%% file: safe_v15_package/appendix_positive_v2_edited.tex
\section{Proofs for Section~\ref{sec:positive-information-recovery}}
\label{app:positive-proofs}

\subsection{Proof of Lemma~\ref{lem:positive-mixture-confidence}}
\label{app:positive-confidence-proof}

\begin{proof}
Define
\[
    M_t:=\sum_{s=1}^t\frac{a_s\xi_s}{\sigma^2}.
\]
Then
\[
    \widehat\theta_t-\theta=\frac{M_t}{V_t}.
\]
Because $a_t$ is predictable and
$\xi_t\sim\mathcal N(0,\sigma^2)$ is independent of the past, for every fixed
$\lambda\in\mathbb R$,
\[
    L_t(\lambda)
    :=
    \exp\!\left(
      \lambda M_t-\frac{\lambda^2}{2}V_t
    \right)
\]
is a nonnegative martingale with initial value one.

Integrate $L_t(\lambda)$ against the Gaussian density
$\lambda\sim\mathcal N(0,\nu^{-1})$.  Completing the square gives the
nonnegative martingale
\[
    \mathcal M_t
    =
    \sqrt{\frac{\nu}{V_t+\nu}}
    \exp\!\left(
      \frac{M_t^2}{2(V_t+\nu)}
    \right).
\]
Ville's inequality implies
\[
    \mathbb P_\theta\!\left(
      \sup_{t\ge1}\mathcal M_t\ge\frac1\delta
    \right)
    \le\delta.
\]
On the complementary event, simultaneously for all $t\ge1$,
\[
    M_t^2
    \le
    2(V_t+\nu)
    \log\!\left(
      \frac{\sqrt{1+V_t/\nu}}{\delta}
    \right).
\]
Dividing by $V_t^2$ proves
\[
    |\widehat\theta_t-\theta|
    \le
    b(V_t,\delta;\nu)
\]
for every $t\ge1$.
\end{proof}

\subsection{Proof of Theorem~\ref{thm:positive-information-recovery}}
\label{app:positive-theorem-proof}

\begin{proof}
Let
\[
    \mathcal E_\infty
    :=
    \left\{
      |\widehat\theta_s-\theta|
      \le b(V_s,\delta;\nu)
      \text{ for every }s\ge1
    \right\}.
\]
Lemma~\ref{lem:positive-mixture-confidence} gives
$\mathbb P_\theta(\mathcal E_\infty^c)\le\delta$.

On $\mathcal E_\infty$, the selected design is conservative.  At $t=1$,
$\vartheta_1=\underline\theta\le\theta$.  For $t\ge2$,
\[
    \widehat\theta_{t-1}-b(V_{t-1},\delta;\nu)
    \le\theta.
\]
Clipping to $[\underline\theta,\overline\theta]$ therefore preserves
$\vartheta_t\le\theta$.  The same confidence bound also gives
\[
    0\le\theta-\vartheta_t
    \le2b(V_{t-1},\delta;\nu),
    \qquad t\ge2.
\]
If the lower clipping is inactive, this follows directly from the two-sided
confidence interval.  If it is active, then
$\widehat\theta_{t-1}-b(V_{t-1},\delta;\nu)<\underline\theta$ and
$\widehat\theta_{t-1}\ge\theta-b(V_{t-1},\delta;\nu)$, which imply the same
bound.

By Assumption~\ref{ass:positive-repeatability}(P2), every episode is safe and
returns to the common continuation state on $\mathcal E_\infty$.  Hence
\[
    \overline P_{\theta,T}^\pi(\mathsf{Safe}_{\theta,T}^c)
    \le
    \mathbb P_\theta(\mathcal E_\infty^c)
    \le\delta.
\]

For regret, the first episode contributes at most $r_{\max}$.  For $t\ge2$,
define the predictable event
\[
    \mathcal E_{t-1}
    :=
    \left\{
      |\widehat\theta_{t-1}-\theta|
      \le b(V_{t-1},\delta;\nu)
    \right\}.
\]
It is $\mathcal G_{t-1}$-measurable, and
$\mathbb P_\theta(\mathcal E_{t-1}^c)\le\delta$.  On
$\mathcal E_{t-1}$,
\[
    \vartheta_t\le\theta,
    \qquad
    0\le\theta-\vartheta_t
    \le2b(V_{t-1},\delta;\nu).
\]
Assumption~\ref{ass:positive-repeatability}(P3)--(P4) then gives
\[
    \bigl(
      v^\star(\theta)
      -\mathbb E_\theta[r_t\mid\mathcal G_{t-1}]
    \bigr)
    \mathbf 1_{\mathcal E_{t-1}}
    \le
    \omega\!\left(2b(V_{t-1},\delta;\nu)\right)
    \mathbf 1_{\mathcal E_{t-1}}.
\]
On $\mathcal E_{t-1}^c$, the one-episode loss is at most $r_{\max}$.  Taking
expectations yields
\[
    \mathbb E_\theta^\pi[v^\star(\theta)-r_t]
    \le
    \mathbb E_\theta^\pi\!\left[
      \omega\!\left(2b(V_{t-1},\delta;\nu)\right)
    \right]
    +r_{\max}\delta.
\]
Summing over $t=2,\ldots,T$ and adding the first-episode bound proves
\eqref{eq:positive-information-regret-bound}.
\end{proof}

\subsection{Proof of Corollary~\ref{cor:positive-information-growth}}
\label{app:positive-rate-proof}

\begin{proof}
Set $\delta_T=T^{-2}$.  Under
\eqref{eq:positive-information-growth-assumption}, for every $1\le t\le T$,
\[
    b(V_t,\delta_T;\nu)
    \le
    C_0t^{-\alpha/2}\sqrt{\log T},
\]
where $C_0$ depends only on $c_-$, $c_+$, $\alpha$, and $\nu$.  Since
$\omega(r)\le Lr^q$,
\[
    \omega\!\left(2b(V_t,\delta_T;\nu)\right)
    \le
    C_1L(\log T)^{q/2}t^{-\alpha q/2}.
\]
Theorem~\ref{thm:positive-information-recovery} therefore gives
\[
    R_T(\pi;\theta)
    \le
    r_{\max}
    +C_1L(\log T)^{q/2}
      \sum_{t=1}^{T-1}t^{-\alpha q/2}
    +\frac{r_{\max}}{T}.
\]
The three cases in \eqref{eq:positive-rate-taxonomy} follow from the standard
partial-sum bounds for $t^{-\alpha q/2}$.
\end{proof}

\subsection{Proof of Corollary~\ref{cor:positive-uniform-alignment}}
\label{app:positive-uniform-alignment-proof}

\begin{proof}
Uniform alignment gives
\[
    \rho t\le V_t\le\bar\rho t.
\]
Apply Corollary~\ref{cor:positive-information-growth} with
$\alpha=1$, $c_-=\rho$, $c_+=\bar\rho$, and $q=1$.
\end{proof}

\subsection{Proofs for the braking example}
\label{app:positive-braking-proof}

\begin{proof}[Proof of Lemma~\ref{lem:positive-braking-alignment}]
Under maximal braking, the stopping distance from
$v(\vartheta)=\sqrt{2D\vartheta}$ is
\[
    \frac{v(\vartheta)^2}{2\theta}
    =
    D\frac{\vartheta}{\theta}
    \le D.
\]
Thus the vehicle stops before the wall whenever $\vartheta\le\theta$.  The
observation $Y_t=\theta+\xi_t$ has coefficient $a_t=1$ and contributes Fisher
information $1/\sigma^2$.

Also,
\[
\begin{aligned}
    v^\star(\theta)-v(\vartheta)
    &=
    \sqrt{2D}\,
    \frac{\theta-\vartheta}{\sqrt\theta+\sqrt\vartheta} \\
    &\le
    \sqrt{\frac{D}{2\underline\theta}}
    (\theta-\vartheta).
\end{aligned}
\]
This proves \eqref{eq:positive-braking-modulus}.
\end{proof}

\begin{proof}[Proof of Corollary~\ref{cor:positive-braking-recovery}]
Lemma~\ref{lem:positive-braking-alignment} verifies
Assumption~\ref{ass:positive-repeatability} with
\[
    \rho=\bar\rho=\frac1{\sigma^2},
    \qquad
    \omega(r)=\sqrt{\frac{D}{2\underline\theta}}\,r,
    \qquad
    r_{\max}=\sqrt{2D\overline\theta}.
\]
Apply Corollary~\ref{cor:positive-uniform-alignment}.
\end{proof}

\subsection{Proof of Proposition~\ref{prop:positive-information-equals-kl}}
\label{app:positive-kl-proof}

\begin{proof}
Conditioned on $(\mathcal G_{t-1},a_t)$, the two observation laws are
\[
    \mathcal N(\theta a_t,\sigma^2)
    \qquad\text{and}\qquad
    \mathcal N(\theta'a_t,\sigma^2).
\]
Their conditional KL divergence is
\[
    \frac{(\theta-\theta')^2a_t^2}{2\sigma^2}.
\]
The policy does not commit during the $T$ experiments, so
Lemma~\ref{lem:general-stopped-kl} gives
\[
    D_{\mathrm{KL}}\!\left(
      P_{\theta,T}^{\pi,\mathrm{pre}}
      \,\middle\|\,
      P_{\theta',T}^{\pi,\mathrm{pre}}
    \right)
    =
    \frac{(\theta-\theta')^2}{2}
    \mathbb E_\theta^\pi\!\left[
      \sum_{t=1}^T\frac{a_t^2}{\sigma^2}
    \right],
\]
which is \eqref{eq:positive-kl-information-identity}.  Any one uniformly safe
policy gives a lower bound on the profile, proving
\eqref{eq:positive-profile-lower-bound}.  Under uniform $\rho$-alignment,
repeating $\kappa_{\underline\theta}$ is safe for the full interval and gives
$V_T\ge\rho T$, which proves \eqref{eq:positive-linear-profile}.
\end{proof}

%% file: safe_v15_package/appendix_computable_edited.tex
\section{Proofs for Section~\ref{sec:certification}}
\label{app:computable-certificates}

\subsection{Proof of Proposition~\ref{lem:open-loop-reduction}}
\label{app:computable-open-loop-proof}

\begin{proof}
Fix a uniformly safe causal policy.  By
Lemma~\ref{lem:general-stopped-kl} and
Assumption~\ref{ass:deterministic-design}(D1),
\[
\begin{aligned}
    D_{\mathrm{KL}}\!\left(
      P_{\theta,T}^{\pi,\mathrm{pre}}
      \,\middle\|\,
      P_{\theta',T}^{\pi,\mathrm{pre}}
    \right)
    =
    \frac{1}{2\sigma^2}
    \mathbb E_\theta^\pi\!\left[
      \sum_{t<\tau^\pi}
      \bigl(g^\top\zeta_t(U_{1:t})\bigr)^2
    \right].
\end{aligned}
\]
Every realized pre-commitment action prefix extends, by (D3), to a sequence in
$\mathcal U_{\mathrm{core}}^T$.  Since all summands are nonnegative,
\[
    \sum_{t<\tau^\pi}
    \bigl(g^\top\zeta_t(U_{1:t})\bigr)^2
    \le
    \sup_{u_{1:T}\in\mathcal U_{\mathrm{core}}^T}
    \Phi_T(u_{1:T})
\]
on every path.  Taking the supremum over uniformly safe policies gives the
upper bound in \eqref{eq:computable-open-loop}.

Conversely, (D3) states that every sequence in
$\mathcal U_{\mathrm{core}}^T$ can be played open loop by a uniformly safe,
noncommitting policy.  For this policy, the stopped record is the full
horizon record, and the preceding identity gives KL divergence
$\Phi_T(u_{1:T})/(2\sigma^2)$.  Taking the supremum over core sequences proves
the reverse inequality.
\end{proof}

\subsection{Proof of Proposition~\ref{prop:computable-single-budget}}
\label{app:computable-single-budget-proof}

\begin{proof}
Suppose first that
$\ker K_T\not\subseteq\ker M_T$.  Then there is a vector
$z\in\ker K_T$ with $z^\top M_Tz>0$.  For every $c>0$, the vector $cz$ is
feasible, while
\[
    (cz)^\top M_T(cz)=c^2z^\top M_Tz\longrightarrow\infty.
\]
Hence the information program is unbounded.

Now suppose $\ker K_T\subseteq\ker M_T$.  Decompose
$\mathbf u=u_r+u_0$, where
$u_r\in\operatorname{range}(K_T)$ and $u_0\in\ker K_T$.  Since
$M_Tu_0=0$, both the objective and the constraint depend only on $u_r$.  Set
\[
    w:=K_T^{1/2}u_r.
\]
Then
\[
    u_r=K_T^{\dagger/2}w,
    \qquad
    u_r^\top K_Tu_r=\|w\|_2^2,
\]
and
\[
    u_r^\top M_Tu_r
    =
    w^\top K_T^{\dagger/2}M_TK_T^{\dagger/2}w.
\]
Maximizing over $\|w\|_2^2\le\epsilon_T$ gives
\[
    \sup_{\mathbf u^\top K_T\mathbf u\le\epsilon_T}
    \mathbf u^\top M_T\mathbf u
    =
    \epsilon_T\lambda_{\max}(M_T,K_T).
\]
Divide by $2\sigma^2$ and apply
Proposition~\ref{lem:open-loop-reduction}.
\end{proof}

\subsection{Proof of Corollary~\ref{cor:computable-block-stationary}}
\label{app:computable-block-stationary-proof}

\begin{proof}
On the range of $K_T$,
\[
    K_T^{\dagger/2}M_TK_T^{\dagger/2}
    =
    I_T\otimes
    \bigl(K_0^{\dagger/2}M_0K_0^{\dagger/2}\bigr).
\]
Its largest eigenvalue is therefore
$\lambda_{\max}(M_0,K_0)$.  Substituting this equality into
\eqref{eq:computable-spectral} gives
\eqref{eq:computable-block-stationary}.
\end{proof}

\subsection{Proof of Theorem~\ref{thm:computable-uniform-sdp}}
\label{app:computable-uniform-sdp-proof}

\begin{proof}
Fix $T$ and a vector $\mathbf u$ in the quadratic outer core.  By
\eqref{eq:computable-matrix-domination},
\[
\begin{aligned}
    \mathbf u^\top M_T\mathbf u
    &\le
    \sum_{j=1}^{m_T}
    \lambda_{j,T}\mathbf u^\top Q_{j,T}\mathbf u \\
    &\le
    \sum_{j=1}^{m_T}\lambda_{j,T}b_{j,T}
    \le\Lambda.
\end{aligned}
\]
Hence the QCQP value in \eqref{eq:computable-qcqp} is at most $\Lambda$.
Proposition~\ref{lem:open-loop-reduction} then gives
\[
    I_{\mathrm{pre}}(T;\theta,\theta')
    \le
    \frac{\Lambda}{2\sigma^2}
\]
for every $T$.  Taking the supremum over horizons proves
\eqref{eq:computable-uniform-capacity}.
\end{proof}

\subsection{Proof of the labeled aggregation bound}
\label{app:computable-multiclass-bridge-proof}

\begin{proof}
Fix a uniformly safe policy and write
\[
    P_j:=P_{\theta_j,T}^{\pi,\mathrm{pre}},
    \qquad
    \overline P:=\frac1M\sum_{k=1}^MP_k.
\]
Under the uniform prior,
\[
    I(J;S_T^\pi)
    =
    \frac1M\sum_{j=1}^M
    D_{\mathrm{KL}}(P_j\|\overline P).
\]
Convexity of KL divergence in its second argument gives
\[
    D_{\mathrm{KL}}(P_j\|\overline P)
    \le
    \frac1M\sum_{k=1}^M D_{\mathrm{KL}}(P_j\|P_k).
\]
Applying the assumed pairwise bounds and averaging over $j$ yields
\[
    I(J;S_T^\pi)
    \le
    \frac1{M^2}
    \sum_{j=1}^M\sum_{k=1}^M
    \overline\beta_{jk,T}.
\]
Finally, take the supremum over uniformly safe policies to obtain
\eqref{eq:computable-multiclass-average}.
\end{proof}

%% file: main_bibtex_ready.bbl
\begin{thebibliography}{10}

\bibitem{AbbasiYadkori11}
Yasin Abbasi-Yadkori and Csaba Szepesv{\'a}ri.
\newblock Regret bounds for the adaptive control of linear quadratic systems.
\newblock In {\em Proceedings of the 24th Annual Conference on Learning
  Theory}, volume~19 of {\em Proceedings of Machine Learning Research}, pages
  1--26. PMLR, 2011.

\bibitem{DeanConstrained19}
Sarah Dean, Stephen Tu, Nikolai Matni, and Benjamin Recht.
\newblock Safely learning to control the constrained linear quadratic
  regulator.
\newblock In {\em 2019 American Control Conference}, pages 5582--5588. IEEE,
  2019.
\newblock arXiv:1809.10121.

\bibitem{LiDasShammaLi21}
Yingying Li, Subhro Das, Jeff Shamma, and Na~Li.
\newblock Safe adaptive learning-based control for constrained linear quadratic
  regulators with regret guarantees, 2021.
\newblock arXiv:2111.00411.

\bibitem{SchifferJanson25}
Benjamin Schiffer and Lucas Janson.
\newblock Foundations of safe online reinforcement learning in the linear
  quadratic regulator: $\sqrt{T}$-regret, 2025.
\newblock arXiv:2504.18657.

\bibitem{HutchinsonJiangAlizadeh26}
Spencer Hutchinson, Nanfei Jiang, and Mahnoosh Alizadeh.
\newblock Rate-optimal regret for the safe learning-based control of the
  constrained linear quadratic regulator, 2026.
\newblock arXiv:2604.22158.

\bibitem{Chernoff59}
Herman Chernoff.
\newblock Sequential design of experiments.
\newblock {\em The Annals of Mathematical Statistics}, 30(3):755--770, 1959.

\bibitem{NaghshvarJavidi13}
Mohammad Naghshvar and Tara Javidi.
\newblock Active sequential hypothesis testing.
\newblock {\em The Annals of Statistics}, 41(6):2703--2738, 2013.

\bibitem{NitinawaratAtiaVeeravalli13}
Sirin Nitinawarat, George~K. Atia, and Venugopal~V. Veeravalli.
\newblock Controlled sensing for multihypothesis testing.
\newblock {\em IEEE Transactions on Automatic Control}, 58(10):2451--2464,
  2013.

\bibitem{GangradeChenSaligrama24}
Aditya Gangrade, Tianrui Chen, and Venkatesh Saligrama.
\newblock Safe linear bandits over unknown polytopes.
\newblock In {\em Proceedings of the 37th Conference on Learning Theory},
  volume 247 of {\em PMLR}, pages 1755--1795, 2024.

\bibitem{GopalanLS21}
Aditya Gopalan, Braghadeesh Lakshminarayanan, and Venkatesh Saligrama.
\newblock Bandit quickest changepoint detection.
\newblock In {\em Advances in Neural Information Processing Systems},
  volume~34, 2021.

\bibitem{Willems05}
Jan~C. Willems, Paolo Rapisarda, Ivan Markovsky, and Bart L.~M. De~Moor.
\newblock A note on persistency of excitation.
\newblock {\em Systems \& Control Letters}, 54(4):325--329, 2005.

\bibitem{Saligrama12chirp}
Venkatesh Saligrama.
\newblock Aperiodic sequences with uniformly decaying correlations with
  applications to compressed sensing and system identification.
\newblock {\em IEEE Transactions on Information Theory}, 58(9):6023--6036,
  2012.

\bibitem{DeanFoCM20}
Sarah Dean, Horia Mania, Nikolai Matni, Benjamin Recht, and Stephen Tu.
\newblock On the sample complexity of the linear quadratic regulator.
\newblock {\em Foundations of Computational Mathematics}, 20:633--679, 2020.

\bibitem{CohenKorenMansour19}
Alon Cohen, Tomer Koren, and Yishay Mansour.
\newblock Learning linear-quadratic regulators efficiently with only $\sqrt{T}$
  regret.
\newblock In {\em Proceedings of the 36th International Conference on Machine
  Learning}, volume~97 of {\em Proceedings of Machine Learning Research}, pages
  1300--1309. PMLR, 2019.

\bibitem{SimchowitzFoster20}
Max Simchowitz and Dylan~J. Foster.
\newblock Naive exploration is optimal for online {LQR}.
\newblock In {\em Proceedings of the 37th International Conference on Machine
  Learning}, volume 119 of {\em Proceedings of Machine Learning Research},
  pages 8937--8948. PMLR, 2020.

\bibitem{ManiaTuRecht19}
Horia Mania, Stephen Tu, and Benjamin Recht.
\newblock Certainty equivalence is efficient for linear quadratic control.
\newblock In {\em Advances in Neural Information Processing Systems},
  volume~32, 2019.

\bibitem{AswaniEtAl13}
Anil Aswani, Humberto Gonzalez, S.~Shankar Sastry, and Claire Tomlin.
\newblock Provably safe and robust learning-based model predictive control.
\newblock {\em Automatica}, 49(5):1216--1226, 2013.
\newblock arXiv:1107.2487.

\bibitem{Feldbaum60}
A.~A. Feldbaum.
\newblock Dual control theory, parts {I--IV}.
\newblock {\em Automation and Remote Control}, 21, 1960.
\newblock Parts I--IV appeared across volumes 21--22, 1960--1961.

\bibitem{AstromWittenmark}
Karl~J. {\AA}str{\"o}m and Bj{\"o}rn Wittenmark.
\newblock {\em Adaptive Control}.
\newblock Addison-Wesley, 2 edition, 1995.

\bibitem{hjalmarsson2005experiment}
H{\aa}kan Hjalmarsson.
\newblock From experiment design to closed-loop control.
\newblock {\em Automatica}, 41(3):393--438, 2005.

\bibitem{GarciaFernandez15}
Javier Garc{\'i}a and Fernando Fern{\'a}ndez.
\newblock A comprehensive survey on safe reinforcement learning.
\newblock {\em Journal of Machine Learning Research}, 16(42):1437--1480, 2015.

\bibitem{HewingEtAl20}
Lukas Hewing, Kim~P. Wabersich, Marcel Menner, and Melanie~N. Zeilinger.
\newblock Learning-based model predictive control: Toward safe learning in
  control.
\newblock {\em Annual Review of Control, Robotics, and Autonomous Systems},
  3:269--296, 2020.

\bibitem{BrunkeEtAl22}
Lukas Brunke, Melissa Greeff, Adam~W. Hall, Zhaocong Yuan, Siqi Zhou, Jacopo
  Panerati, and Angela~P. Schoellig.
\newblock Safe learning in robotics: From learning-based control to safe
  reinforcement learning.
\newblock {\em Annual Review of Control, Robotics, and Autonomous Systems},
  5:411--444, 2022.

\bibitem{MayneEtAl00}
David~Q. Mayne, James~B. Rawlings, Christopher~V. Rao, and Pierre O.~M.
  Scokaert.
\newblock Constrained model predictive control: Stability and optimality.
\newblock {\em Automatica}, 36(6):789--814, 2000.

\bibitem{blanchini1999invariance}
Franco Blanchini.
\newblock Set invariance in control.
\newblock {\em Automatica}, 35(11):1747--1767, 1999.

\bibitem{fisac2019general}
Jaime~F. Fisac, Anayo~K. Akametalu, Melanie~N. Zeilinger, Shahab Kaynama,
  Jeremy~H. Gillula, and Claire~J. Tomlin.
\newblock A general safety framework for learning-based control in uncertain
  robotic systems.
\newblock {\em IEEE Transactions on Automatic Control}, 64(7):2737--2752, 2019.

\bibitem{wabersich2021predictive}
Kim~P. Wabersich and Melanie~N. Zeilinger.
\newblock A predictive safety filter for learning-based control of constrained
  nonlinear dynamical systems.
\newblock {\em Automatica}, 129:109597, 2021.

\bibitem{ames2019control}
Aaron~D. Ames, Samuel Coogan, Magnus Egerstedt, Gennaro Notomista, Koushil
  Sreenath, and Paulo Tabuada.
\newblock Control barrier functions: Theory and applications.
\newblock In {\em 2019 European Control Conference}, pages 3420--3431. IEEE,
  2019.

\bibitem{MoldovanAbbeel12}
Teodor~Mihai Moldovan and Pieter Abbeel.
\newblock Safe exploration in markov decision processes.
\newblock In {\em Proceedings of the 29th International Conference on Machine
  Learning}, pages 1711--1718, 2012.
\newblock arXiv:1205.4810.

\bibitem{turchetta2016safe}
Matteo Turchetta, Felix Berkenkamp, and Andreas Krause.
\newblock Safe exploration in finite markov decision processes with gaussian
  processes.
\newblock In {\em Advances in Neural Information Processing Systems},
  volume~29, 2016.

\bibitem{KollerEtAl19}
Torsten Koller, Felix Berkenkamp, Matteo Turchetta, Joschka Boedecker, and
  Andreas Krause.
\newblock Learning-based model predictive control for safe exploration and
  reinforcement learning, 2019.
\newblock arXiv:1906.12189.

\bibitem{berkenkamp2017safe}
Felix Berkenkamp, Matteo Turchetta, Angela~P. Schoellig, and Andreas Krause.
\newblock Safe model-based reinforcement learning with stability guarantees.
\newblock In {\em Advances in Neural Information Processing Systems},
  volume~30, 2017.

\bibitem{sui2015safe}
Yanan Sui, Alkis Gotovos, Joel Burdick, and Andreas Krause.
\newblock Safe exploration for optimization with gaussian processes.
\newblock In {\em Proceedings of the 32nd International Conference on Machine
  Learning}, volume~37 of {\em Proceedings of Machine Learning Research}, pages
  997--1005. PMLR, 2015.

\bibitem{amani2019linear}
Sanae Amani, Mahnoosh Alizadeh, and Christos Thrampoulidis.
\newblock Linear stochastic bandits under safety constraints.
\newblock In {\em Advances in Neural Information Processing Systems},
  volume~32, 2019.

\bibitem{dobbe2020learning}
Roel Dobbe, Patricia Hidalgo-Gonzalez, Stavros Karagiannopoulos, Rodrigo
  Henriquez-Auba, Gabriela Hug, Duncan~S. Callaway, and Claire~J. Tomlin.
\newblock Learning to control in power systems: Design and analysis guidelines
  for concrete safety problems.
\newblock {\em Electric Power Systems Research}, 189:106615, 2020.

\bibitem{vanWaarde20}
Henk~J. van Waarde, Jaap Eising, Harry~L. Trentelman, and M.~Kanat Camlibel.
\newblock Data informativity: A new perspective on data-driven analysis and
  control.
\newblock {\em IEEE Transactions on Automatic Control}, 65(11):4753--4768,
  2020.

\bibitem{VenkateshDahleh97}
S.~R. Venkatesh and M.~A. Dahleh.
\newblock Identification in the presence of classes of unmodeled dynamics and
  noise.
\newblock {\em IEEE Transactions on Automatic Control}, 42(12):1620--1635,
  1997.

\bibitem{VenkateshDahleh01}
S.~R. Venkatesh and M.~A. Dahleh.
\newblock On system identification of complex systems from finite data.
\newblock {\em IEEE Transactions on Automatic Control}, 46(2):235--257, 2001.

\bibitem{saligrama2005convex}
Venkatesh Saligrama.
\newblock A convex analytic approach to system identification.
\newblock {\em IEEE transactions on automatic control}, 50(10):1550--1566,
  2005.

\end{thebibliography}
